\documentclass[10pt,a4paper]{elsarticle}
\usepackage{amsthm}
\usepackage{amsmath}
\usepackage{bbm}

\usepackage{amsfonts}
\usepackage{amssymb}
\usepackage{mathtools}
\usepackage{graphicx}
\usepackage{natbib}
\usepackage{caption}
\usepackage{subcaption}
\usepackage{dutchcal}

\usepackage{lmodern}
\usepackage{amsfonts}
\usepackage{tikz-cd}
\usepackage{hyperref}

\usepackage[shortlabels]{enumitem}

\renewcommand{\Bbb}{\mathbb}
\newcommand{\N}{\Bbb{N}}
\newcommand{\Z}{\Bbb{Z}}
\newcommand{\Q}{\Bbb{Q}}
\newcommand{\R}{\Bbb{R}}
\newcommand{\C}{\mathcal{C}}
\renewcommand{\L}{\mathcal{L}}

\renewcommand{\P}{\mathcal{P}}

\newcommand{\E}{\mathcal{E}}
\newcommand{\J}{\mathcal{J}}
\newcommand{\V}{\mathcal{V}}
\newcommand{\Emb}{\mathcal{Emb}}

\newcommand{\colim}{\operatorname{colim}}

\renewcommand{\phi}{\varphi}
\usepackage{amsmath}
\usepackage{amsfonts}
\usepackage{amssymb}
\usepackage{graphicx}
\usepackage[left=2cm,right=2cm,top=2cm,bottom=2cm]{geometry}

\newcommand{\Poly}{\operatorname{Poly}}

\newcommand{\Set}{\operatorname{\textbf{Set}}}
\newcommand{\Conf}{\operatorname{Conf}}

\newtheorem{theorem}{Theorem}[section]
\newtheorem{cor}[theorem]{Corollary}
\newtheorem{lemma}[theorem]{Lemma}

\newtheorem{exa}[theorem]{Example}

\newtheorem{definition}[theorem]{Definition}

\newtheorem{re}[theorem]{Remark}

\title{Transparent Semantic Spaces: A Categorical Approach to Explainable Word Embeddings}
\begin{document}
\begin{frontmatter}

\author{Ares Fabregat-Hernández\corref{cor1}$^{1,2}$}
\ead{arfabher@upv.edu.es}
\cortext[cor1]{Corresponding author}

\author{Javier Palanca$^{1}$}
\ead{jpalanca@dsic.upv.es}
\author{Vicent Botti$^{1,3}$}
\ead{vbotti@dsic.upv.es}
\address{$^{1}$ Valencian Research Institute for Artificial Intelligence (VRAIN)\\ Universitat Politècnica de València, \\
	Camí de Vera s/n, 46022, Valencia,  Spain}
\address{$ ^{2} $Universidad Internacional de Valencia (VIU), C/Pintor Sorolla 21, 46002 Valencia, Spain}
\address{$^{3}$ valgrAI (Valencian Graduate School and Research Network of Artificial Intelligence)}

\begin{abstract}
	The paper introduces a novel framework based on category theory to enhance the explainability of artificial intelligence systems, particularly focusing on word embeddings. Key topics include the construction of categories $ \L_{T} $ and $ \P_{T} $, providing schematic representations of the semantics of a text $ T $, and reframing the selection of the element with maximum probability as a categorical notion. Additionally, the monoidal category $ \P_{T} $ is constructed to visualize various methods of extracting semantic information from $ T $, offering a dimension-agnostic definition of semantic spaces reliant solely on information within the text.
	
	Furthermore, the paper defines the categories of configurations $ \Conf $ and word embeddings $ \Emb $, accompanied by the concept of divergence as a decoration on $ \Emb $. It establishes a mathematically precise method for comparing word embeddings, demonstrating the equivalence between the GloVe and Word2Vec algorithms and the metric MDS algorithm, transitioning from neural network algorithms (black box) to a transparent framework. Finally, the paper presents a mathematical approach to computing biases before embedding and offers insights on mitigating biases at the semantic space level, advancing the field of explainable artificial intelligence.

\end{abstract}


\end{frontmatter}

\section{Introduction}

Word embeddings have emerged as a cornerstone in natural language processing (NLP) and machine learning (ML) applications, revolutionizing the representation of textual data (see \citep{incitti2023beyond}). At the heart of word embeddings lies the idea of capturing semantic relationships between words in a continuous vector space, enabling machines to understand and process human language more effectively (see \citep{hashimoto2016word, levy2014neural}). By mapping words to high-dimensional vectors, word embeddings encode semantic similarities and syntactic structures, thereby facilitating a wide array of downstream tasks such as sentiment analysis, named entity recognition, machine translation, and document classification. In addition to enhancing model performance and accuracy, word embeddings offer several practical advantages in ML applications. They provide a compact and dense representation of textual data, enabling efficient storage, retrieval, and computation. Moreover, word embeddings capture contextual nuances and semantic meanings that traditional bag-of-words or one-hot encoding schemes fail to capture, leading to more nuanced and context-aware language understanding. As such, word embeddings serve as foundational building blocks for a broad spectrum of ML tasks, empowering researchers and practitioners to unlock new capabilities in language understanding and processing. In recent years, word embeddings have become indispensable tools for natural language processing tasks, offering compact representations of textual data that capture semantic relationships between words. 

However, the design and interpretation of word embeddings present several challenges. Biases embedded in the training data can be perpetuated in word embeddings, leading to unfair associations and stereotypes. Moreover, the choice of embedding dimensionality poses a trade-off between capturing nuanced semantic relationships and computational efficiency. Additionally, the opacity of the embedding process, particularly in neural network-based approaches, can hinder interpretability and trustworthiness (see \citep{caliskan2022gender}).

Addressing these challenges requires interdisciplinary efforts from researchers in machine learning, natural language processing, and ethics. Strategies such as debiasing techniques, dimensionality reduction methods, and transparency-enhancing approaches are being actively explored to mitigate these challenges and improve the reliability and fairness of word embeddings in practical applications.

Explainable artificial intelligence (XAI) is indispensable for enhancing the transparency and interpretability of word embeddings, tackling the inherent opacity associated with the embedding process. As word embeddings become increasingly integral to real-world applications, comprehending how they derive semantic relationships and representations becomes imperative to ensure their trustworthiness and reliability (see \cite{Survey}). Techniques for explainability offer valuable insights into the underlying mechanisms of word embedding models, illuminating the encoding and processing of linguistic features. By rendering the embedding process interpretable, XAI empowers stakeholders to assess the robustness of word embeddings, detect potential biases, and validate the integrity of semantic relationships. Moreover, explainability fosters accountability and trust in AI systems by enabling users to comprehend and scrutinize the decisions made by word embedding models, thereby advocating ethical and responsible deployment across domains like natural language processing, information retrieval, and content recommendation.

In response to these challenges, we propose a novel framework based on category theoretical properties, as in \citep{CatSem,BPF}, for constructing word embeddings and analyzing semantic spaces, which presents several notable advantages over conventional methods. Firstly, our approach defines semantic spaces independently of dimension and added structure, offering greater flexibility and adaptability in capturing semantic relationships. This dimension-neutral strategy allows for exploring semantic spaces without being constrained by fixed dimensionality.

Secondly, we provide transparent mathematical formulas to elucidate the embedding process for two of the most used word embeddings the GloVe \cite{pennington2014glove} and the Word2Vec \cite{Mikolov2013EfficientEO} algorithms, offering insights into the underlying mechanisms without relying on ``hidden variables" or opaque transformations. This transparency enhances interpretability and facilitates the identification of biases inherent in the embedding process. By explicitly defining the embedding process, our method enables researchers and practitioners to understand and control the factors influencing the resulting embeddings.

Lastly, our framework facilitates the modification of semantic spaces to mitigate potential biases present in textual data. By equipping users with tools for bias reduction and modification, our approach promotes fairness and inclusivity in natural language processing applications. We believe that our dimension-neutral, transparent approach represents a significant step forward in the development of interpretable, bias-aware word embeddings. By offering researchers and practitioners the means to understand and harness the semantic structure of textual data, our method holds promise for advancing natural language processing, machine learning, and related fields.

The subsequent sections of this paper are structured as follows: Section \ref{sec:Spaces} introduces three distinct spaces, namely $ X_{T}$, $\C_{T}$, and $\L_{T}$, which serve as visualization aids for understanding semantic relationships and reframing certain machine learning (ML) concepts within a categorical framework. In Section \ref{sec:P_T}, we consolidate the structures present in $\L_{T}$ to form a novel category, where morphisms encode semantic information of sets of expressions. Leveraging this characteristic, we define semantic spaces towards the section's conclusion.

Building upon the groundwork laid in Section \ref{sec:P_T}, Section \ref{sec:WE} delineates the categories of configurations and word embeddings, culminating in the central assertion of this paper: the equivalence between GloVe and Word2Vec neural word embeddings and metric MDS embeddings. Finally, the concluding section offers a synthesis of the findings, outlines future research avenues, and encapsulates the paper's key takeaways.

Finally, we have included two appendices after the manuscript. The first appendix provides the proof of the results outlined in the paper. The second appendix offers comprehensive details regarding the construction of weighted limits and colimits in enriched category theory.

\section{The spaces $X_{T}$, $\mathcal{C}_{T}$ and $\mathcal{L}_{T}$}\label{sec:Spaces}
In \citep{CatSem} the authors introduced the syntax and semantic categories based on a text. They are categories enriched over the unit interval to represent the statistical information present on a given text $T$. In this section, we introduce those categories and fix some notations.

\subsection{The Alexandrov Space $X_{T}$}

Given a text $T$ yields a poset $(X_{T},\leq)$, that is a set with a partial order. The set $X_{T}$ is the set of all expressions found in $T$ (that we also call $n$-grams of $T$). The preorder $\leq$ is given by $x\leq y$ if $x$ is a sub-expression of $y$. We define the \textbf{length of $x$, $ \ell(x)$} as the number of words in contains. This space can be endowed with the Alexandrov topology. This means that for every $x\in X_{T}$ we have a smallest neighbourhood:

\begin{equation*}
	U_{x}=\{t\in X_{T}\colon x\leq t\},
\end{equation*}

that is the set of all sequences in $T$ containing $x$. Conceptually, this is the set to take into account to compute conditional probabilities of extensions. We can refine those sets by filtering by the length of the sequences:

\begin{equation*}
	U^{n}_{x}=\{t\in X_{T}\colon x\leq t,\ \ell(t)= n \}.
\end{equation*}

In this case, we are only looking at sequences of length $n$ that contain $x$. Finally, we can grade the space $X_{T}$ using the lengths of sequences:

\begin{equation}
	X_{T}=\bigoplus_{n\geq 1}X_{T}^{n},
\end{equation}
where $X_{T}^{n}$ is the set of sequences of length $n$. In this way, the set $X_{T}^{1}$ is the set of words that appear in $T$, also called the dictionary of $T$, the set $X_{T}^{2}$ is the set of bigrams of $T$ and so on. Thus, we can rewrite the sets $U_{x}^{\leq n}$ as 

\begin{equation*}
	U_{x}^{\leq n}=\mathcal{F_{n}}(U_{x})=\bigoplus_{m\leq n} \left( U_{x}\cap X_{T}^{m}\right).
\end{equation*}

The collection $\{U_{x}^{n}\}$ form a neighborhood system of $x$ for a pretopology on $X_{T}$. 

\begin{figure}
	\centering
	\includegraphics[scale=0.2]{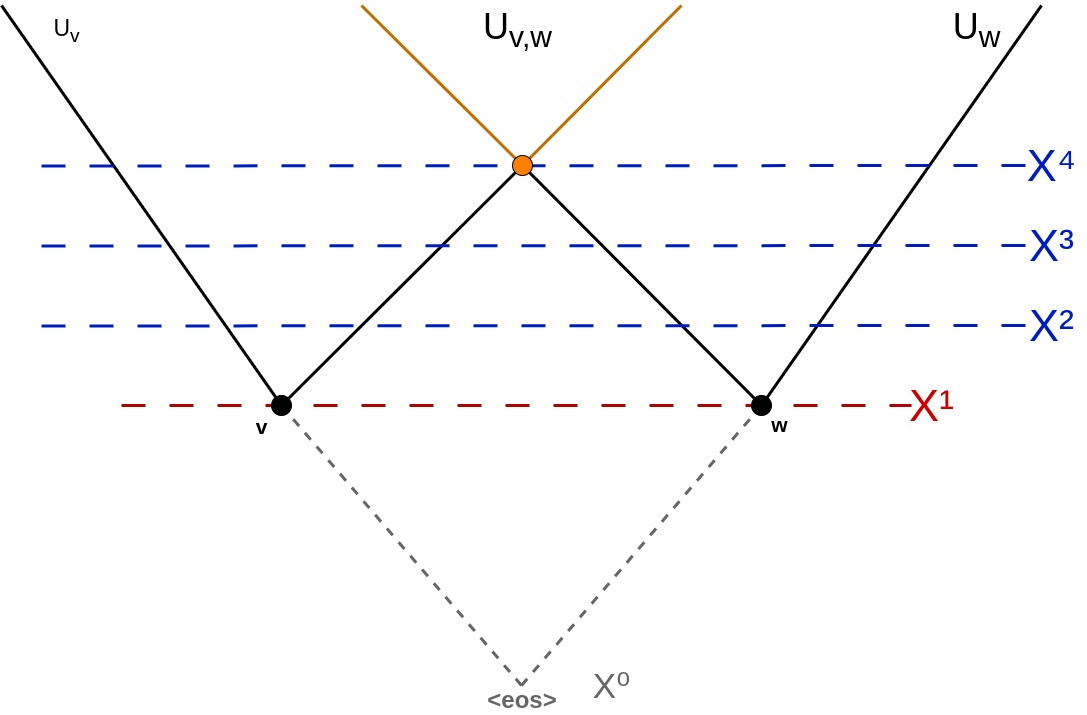}
	\caption{Alexandrov Cones on $X_{T}$.}
	\label{fig:ACbien}
\end{figure}

Lastly, most of the algorithms use a special symbol ``<eos>" when counting the occurences. We can incorporate this as the $0$-th graded piece of $X_{T}$, that is $X_{T}^{0}=\{<eos>\}$. With this, we can picture the space 

\begin{equation}
	X_{T}=\bigoplus_{n\geq 0}X_{T}^{n}.
\end{equation}

as a series of cones emanating from each word in $X_{T}^{1}$ and rays emanating from $X_{T}^{0}$ to each word (see figure \ref{fig:ACbien}\footnote{We have omitted the subscript $T$ in the figure to improve readability.}). Two cones emanating from two words $x$ and $y$ are going to intersect at $X_{T}^{m}$ with $m$ the minimum length of expression that contains both $x$ and $y$. Thus this gives a first measure of closeness: the closer two words are the lower the intersection of the cones is. Of course, this picture is imperfect, there might not only be only one expression that contains both words but several of them. Thus the point should be higher dimensional (eg. a line). Furthermore, this picture (and by extension the Alexandrov space) does not take into account frequency. This means that if there are expressions of length $n$ that contain $x$ and $y$, and other expressions of the same length that contain $x$ and $z$ we cannot decide if $y$ or $z$ is closer to $x$. This is a reflection of the fact that it is a topological space and frequency implies some kind of metric. 

Furthermore, we cannot measure distances between two elements if we restrict ourselves to a graded piece $X_{T}^{n}$. Since the ultimate goal of a word embedding algorithm is to give (linear) structure to $X_{T}^{1}$ based on probabilistic information, we need to add more structure to the spaces. This will be done by upgrading our construction of the space $X_{T}$ to a category-theoretic object $\C_{T}$.

\subsection{The Syntax Category $\mathcal{C}_{T}$}

Given a poset $(X_{T},\leq)$ we can construct a category $C_{T}$ named the \textbf{syntax category of $T$}. Its objects are expressions found in $T$ and morphisms are given by the preorder $x\leq y$. That is, there is a morphism $\C_{T}(x,y)$ if and only if $y$ is an extension of $x$. Thus, the set $\C_{T}(x,y)$ is either empty or contains a single element. The unit interval $([0,1],\leq)$ is also a preorder where we can multiply two elements we get a symmetric monoidal category $([0,1],\leq,\otimes)$, where $a\otimes b = ab$ for all $a, b\in [0,1]$.

Thus we can enrich the syntax category $\C_{T}$ over $[0,1]$ turning each $\hom$-set in an object of $[0,1]$. In our case, we want those objects to be conditional probabilities. Precisely, this means:

\begin{equation}\label{eq:def}
	\C_{T}(x,y)=\left\{\begin{matrix}
		p(y|x),&\text{if $x\leq y$;}  \\
		0,& \text{otherwise,}  \\
	\end{matrix}\right.
\end{equation}

Notice that in this setting $\C_{T}(x,x)=1$ and we have a probabilistic triangle inequality

\begin{align*}
	\C_{T}(x,y)\otimes\C_{T}(y,z)\to&\C_{T}(x,z)\\
	p(y|x)p(z|y)\leq& p(z|x).
\end{align*}

\begin{exa}\label{exa:probs}
	\begin{enumerate}
		\item Given a text $T$ the most basic question is what is the probability that a word $w$ is next given a given history $g^{-}$. This can be seen as the probability $p(g^{-}w|g^{-})$. If we let the probability come from counts, i.e.
		
		\begin{equation}
			\C_{T}(g^{-},g^{-}w)=p(g^{-}w|g^{-})=\frac{C(g^{-}w)}{C(g^{-})}
		\end{equation}
		
		we recover the $n$-gram probability distribution with $n=\ell(g)$ the length of $g$.
		
		\item We can do the same procedure but in reverse: give the probability that a word $w$ precedes a certain $n$-gram $g^{+}$:
		
		\begin{equation}
			\C_{T}(g^{+},wg^{+})=p(wg^{+}|g^{+})=\frac{C(wg^{+})}{C(g^{+})}
		\end{equation}
	\end{enumerate}
\end{exa}
The limitation of this framework is that the expression must be in the text. This means that to compute the probability of a word $w$ given a window-based context $g^{\pm}$ (this means that $t=g^{-}wg^{+}$ is an object in $\C_{T}$) we would need some extra structure since $g^{\pm}$ is not an object in $\C_{T}$ and hence we cannot compute $\C_{T}(g^{\pm},t)$.
We would like to combine the preceding examples to get the probability

\begin{equation}\label{eq:pw2v}
	p(w|g^{\pm})=\frac{C(g^{-}wg^{+})}{\sum_{v\in \C_{T}^{1}}C(g^{-}v	g^{+})}
\end{equation}

To do that we have to introduce a grading in $\C_{T}$ so we can define $\C_{T}^{1}$ of the equation. There is a concept of graded monads\footnote{\url{https://ncatlab.org/nlab/show/graded+monad}} but the problem with this concept is that it is based on endofunctors (a functor $F\colon \C\to\C$) indexed by a monoidal category, and in our case we want to generate subcategories of $\C_{T}$ graded by the length of the expression.  

This means that we have a functor from $\C_{T}$ to the category $(\N,\leq)$:

\begin{equation}
	\ell \colon \C_{T}\to \N
\end{equation}

with $\ell(g)$ being the length of $g$. With that we can define the $n$-th graded piece of $\C_{T}$ as the fiber of $n\in \N$: 

\begin{equation}
	\C_{T}^{n}=\ell^{-1}(n),
\end{equation}

Those are precisely the expressions of $T$ with length $n$. In particular we have that $\C_{T}^{0}=\text{<eos>}$ and $\C_{T}^{1}$ are the words found in $T$. Thus we can recover the Alexandrov cones of Figure \ref{fig:ACbien} by starting from an element $w$ of $\C_{T}^{1}$ and then, taking all the elements $g$ of $\C_{T}^{2}$ such that $\C_{T}(w,g)\neq\{0\}$ and repeating the process for each graded piece.

Furthermore, a necessary condition for the existence of an element in $C_{T}(g,t)$, with $g$ in $\C_{T}^{m}$ and $t$ in $\C_{T}^{n}$, is that $m < n$. Indeed, if $m\leq n$ an expression of length $n$ may be extended by an expression of length $m$. Thus in $\C_T$, there are no morphisms between objects of the same graded piece. This means that we have no direct way of comparing two objects of the same length. For that, we are going to introduce a bigger category denoted $\L_{T}$ with the same objects as $\C_{T}$ but more possible morphism between objects.

\subsection{The Category $\mathcal{L}_{T}$}

First of all, we need to recall a result from category theory called the Yoneda embedding.

\begin{lemma}\label{lem:yoneda}
	Let $\mathcal{C}$ be a category and $x,y$ objects in $\mathcal{C}$. Then $x$ is isomorphic to $y$ if and only if $\C(x,-)$ is isomorphic to $\C(y,-)$.
\end{lemma}

This means that we can substitute the object of the category by setting the morphisms from that object to all the objects. Since we are working with enriched categories we would need the enriched version of the Yoneda embedding, where the isomorphism is taken in the category of enrichment and not merely in the category of sets. This is reminiscent of the fact that any (locally small) category is a category enriched over the category of sets. Since the statement of the enriched version of the embedding is practically identical to lemma \ref{lem:yoneda} we refer the reader to \cite[Section 2.4]{kelly1982basic} for a precise statement and proof.

\begin{exa}\label{exa:YE}
	Let $T$ be a text and $\C_{T}$ its associated category. Then, given a word $w$ the Yoneda embedding produces a functor $\C_{T}(w,-)\colon \C_{T}\to [0,1]$ given by $g\mapsto \C_{T}(x,g)=p(g|x)$. 	In other words, we are replacing the word $w$ (and its position in $\C_{T}$) with its set of all the conditional probabilities of extensions.
\end{exa}

By example \ref{exa:YE}, we see that we can now compare two words by comparing their embeddings in the functor category. That is, we can define the probability 

\begin{equation}\label{eq:endoL}
	p(v\| w)=\inf_{g\in\C_{T}}\left\{\frac{p(g|v)}{p(g|w)},1\right\}
\end{equation}

with $v,w\in \C_{T}^{1}$. This yields the following definition.

\begin{definition}
	
	Let $\L_{T}$ be the enriched category whose objects are the same as in $\C_{T}$, morphisms between objects of different graded pieces are as in $\C_{T}$, and morphisms between objects in the same graded piece are given by
	
	\begin{equation*}\label{eq:samemorphism}
		p(a\| b)=\inf_{g\in\C_{T}}\left\{\frac{p(g|a)}{p(g|b)},1\right\}.
	\end{equation*}
\end{definition}

\begin{figure}
	\centering
	\includegraphics[scale=0.3]{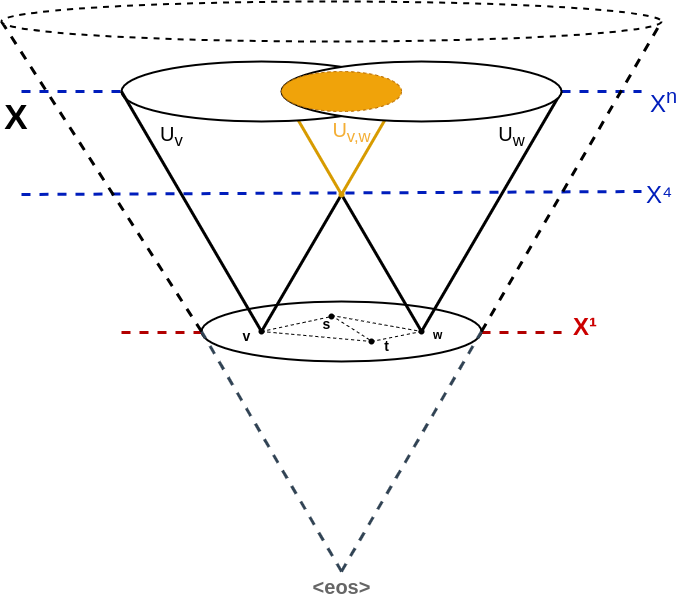}
	\caption{Alexandrov Cones on $L_{T}$.}
	\label{fig:AlexandrovConesLT}
\end{figure}

If $a$ and $b$ are similar, that is, they appear in similar contexts and with similar frequency we can say that they have comparable meanings. This is represented by the probability $p(a\|b)$ which can be interpreted as the likelihood that we could use both expressions interchangeably (using the text $T$ as a reference).

\begin{re}\label{re:notation}
	We will use the notation $p(a\| b)$ to signify $p(a|b)$ if $\ell(b)<\ell(a)$ and $p(a\| b)$ if $\ell(b)=\ell(a)$.
\end{re}

Figure \ref{fig:AlexandrovConesLT} shows what we have accomplished by passing from $\C_{T}$ to $\L_{T}$. As shown by the dashed lines on $X^{1}$ we can now compare elements on the same graded piece of $\L_{T}$. This effectively ``adds dimensions" to the representation of the category. Furthermore, this allows us to look at each $\L_{T}^{n}$ as a space on its own which will be crucial to our definition of semantic spaces later on.

Another thing that we can deduce from figure \ref{fig:AlexandrovConesLT} is that $U_{v,w}=U_{g}$ with $g$ the shortest expression that contains both $v$ and $w$. This in turn yields that expressions that contain $g$ will be ``closer" than expressions that contain $g'<g$ but not $g$.

With that in mind, we can start exploring some properties of this category. The first topic to consider is limits and colimits. Given that we are operating within an enriched framework, the definitions and properties can be rather intricate in terms of notation. We defer the precise definitions to Appendix A. The main idea is that limits and colimits are going to formalize the concepts of maximum probability given a context.

\begin{lemma}\label{lem:aux}
	Let $T$ be a text and $\L_{T}$ its syntax category and let $\textbf{2}=\{1,2\}$ be the category with two objects and only identity morphism. Let $W\colon\textbf{2}\to [0,1]$ be a functor of weights. Then colimit of the diagram of $F\colon\textbf{2}\to \L_{T}$ is an object $g^{*}=\colim{}^{W}F$ in $\L_{T}$ such that
	
	\begin{equation}\label{eq:colimminT}
		p(t\|g^{*}) = \min\left\{\frac{p(t\|F(1))}{W(1)},\frac{p(t\|F(2))}{W(2)},1 \right\} 
	\end{equation}
	
\end{lemma}

The proof can be found in Appendix \ref{sec:AppendixA}, Lemma \ref{lem:auxApp}. Now, in equation \eqref{eq:colimminT} the role of the functor of weights is still a little unclear. Let's assume that, for a given pair $F(1)=g^{-}$ and $F(2)=g^{+}$ of expressions and a pair of weights $w_{i}$ with $i=1,2$, the minimum is satisfied by $p(t\|g^{+})/w_2$. Then by the universal property of the (weighted) colimit, we have that

\begin{equation*}
	p(t\|g^{*})=\frac{p(t\|F(1))}{w_2}.
\end{equation*}

Furthermore, by the composition rules on an enriched category, we have:

\begin{equation*}
	p(t\|g^{+})\geq p(g^{*}\|g^{+})\cdot p(t\|g^{*}),
\end{equation*}

thus yielding $w_2\geq p(g^{*}\|g^{+})$. This means that the weights bound the conditional probabilities of the expressions by the contexts. The safest approach to the theory would be to set $w_i=1$ for $i=1,2$ which would remove the effect of the weights. If we let some of the weights be different from one, the effect would be to limit the importance we are giving to a certain context of the colimit.

It is worth to analyze what $g^{*}=\colim^{W}F$ might be. First of all, $g^{*}$ has to contain both $F(1)=g^{-}$ and $F(2)=g^{+}$ which we think of as contexts of $g^{*}$. By definition of the morphisms in $\L_{T}$, if we assume that $g^{+}\neq g^{-}$ and the end of $g^{-}$ does not coincide with the begging of $g^{+}$, the colimit is at least in $\L_{T}^{n}$ with $n = \ell(g^{-})+\ell(g^{+})$. Indeed the colimit will be the most likely expression to contain both contexts and it might be $g=g^{-}g^{+}$. However, if we choose the context expression as in \cite{Mikolov2013EfficientEO}, we can assume the colimit is of the form: $g^{*}=g^{-} w g^{+}$ in $\L_{T}^{2k+1}$ with $k=\ell(g^{-}) = \ell(g^{+})$. With this, we have the following result.

\begin{theorem}\label{thm:colimmaxprob}
	With the same hypothesis of the preceding lemma, the colimit of the diagram $g^{*}=\colim ^{W}F$ satisfies the following equation:
	
	\begin{equation}\label{eq:colimmaxprobT}
		g*= \max_{t\in \L_{T}}p(t\| g^{\pm})
	\end{equation}
	with $t=g^{-}wg^{+}\in \L_{T}^{2k+1}$ and
	
	\begin{equation}\label{eq:probW2V}
		p(t\| g^{\pm})=\frac{C(t)}{\sum_{v\in\L_{T}^{1}}C(g^{-}vg^{+})}.
	\end{equation}
\end{theorem}

The proof can be found in Appendix \ref{sec:AppendixA}, Theorem \ref{thm:colimmaxprobA}. Notice that the probability of equation \ref{eq:probW2V} is essentially the one recovered by the Word2Vec algorithm in \cite{Mikolov2013EfficientEO} as can be seen in \cite{levy2014neural}. Thus, under the assumptions of Lemma \ref{lem:aux}, the colimit of the diagram $F\colon \textbf{2}\to \L_{T}$  yields the expression of length $2k+1$ with maximum conditional probability. Equivalently, the the colimit gives the middle word $v$ such that $g^{*}=g^{-}vg^{+}$ has maximum probability. 

Lastly, notice that we are assuming that the colimit is of the form $g^{*}=g^{-}vg^{+}$ and not $g^{*}=g^{+}vg^{-}$ (contexts are in the opposite order). This is because language is non-commutative and it is expected that $p(g^{-}vg^{+}|g^{\pm})\geq p(g^{+}vg^{-}|g^{\pm})$. This assumption is implicitly in \cite{Mikolov2013EfficientEO} since the algorithm (CBOW in this case) sums all of the coordinates of the one-hot encoding of the context words which renders the ordering moot and still produces the word with greater probability.

All the theory thus far requires to work with a fixed (conditional) probability distribution which means that there is only one way to compute $p(t|g)$ in $\L_{T}$. Any other way would produce a different category $\L^{\prime}_{T}$. This is a limitation when we want to work with different probabilities and different spaces to study bias or convergence of approximations to a certain probability distribution. To incorporate these new contributions into our framework we need to introduce Markov categories and a slight generalization of them.

\section{The Category $\P_{T}$}\label{sec:P_T}

In this section, we want to aggregate all the information of all different categories $\L_{T}$. For that, we construct the category $\P_{T}$. The objects in $ \P_{T} $ are going to be sets of expressions in the text $T$ and their Cartesian products. The morphisms are going to matrices whose entries represent probabilities satisfying certain conditions. An important subclass of morphisms is going to be the class of row-stochastic matrices. This means that a morphism $f\colon X \to Y$ in $\P_{T}$ is a matrix with non-negative entries, columns indexed by elements of $X$, and rows indexed by elements of $Y$ with $f(x,y)\in [0,1]$. We call this kind of matrices \textbf{probabilistic matrices} and the probabilities found within \textbf{transition probabilities}. If furthermore, the matrix satisfies for all $x\in X$,

\begin{equation}\label{eq:MCatprob}
	\sum_{y\in Y}f(x,y)=1.
\end{equation}

we call the matrix row-stochastic and we write $f(x,y)=f(y|x)$.
The interpretation of equation \eqref{eq:MCatprob} is that each $f(y|x)$ represents the conditional probability of $y\in Y$ given $x\in X$. In particular, for any morphism 

\begin{equation}
	p\colon \L_{T}^{0}\to X
\end{equation}

we have $p(x \ |<eos>)=p(x)$. This means that morphisms from the $0$-th graded piece yield probability distributions on $X$.

We have already seen examples of objects and morphisms of $\P_{T}$, namely the graded pieces of $\L_{T}$: for each $n\in\N$ the set of all expressions of length $n$ of $T$, $\L_{T}^{n}$, is an object in $\P_{T}$. There are however more objects than those specific sets. For instance, given a word $w$ we can consider the set:
\begin{equation}
	\pi^{-1}_{2k+1}(w)=\{g\in \L_{T}^{2k+1}\colon g=g^{-}wg^{+}\}\subset \L_{T}^{2k+1}.
\end{equation}

This set comes from considering the mapping

\begin{equation}
	\pi_{2k+1}\colon\L_{T}^{2k+1}\to\L_{T}^{1}
\end{equation}
where every expression $t\in \L_{T}^{2k+1}$ is mapped to its middle word with probability $1$ and to any other word with probability $0$.  This showcases one key difference between this category and the ones treated above: we can have morphisms between any two objects in the category $\P_{T}$ whereas in the previous cases, morphisms were limited to extensions of expressions and comparisons between expressions (equation \eqref{eq:samemorphism}).

Another important difference is that in the previous categories, you had to choose one way to compute the probabilities and then the composition law implied some restrictions on how the probabilities would be. This is not the case anymore which allows us to consider more situations (probability metrics) all at once. Before moving on, it can be helpful to revisit some of the examples in the previous section.

\begin{exa}
	\begin{enumerate}
		\item The $n$-gram probabilities from example \ref{exa:probs} can be understood as  a sub-matrix of the morphism
		$ P\colon\L_{T}^{n}\to\L_{T}^{n+1}$ with $P(g,t)=p(t|g)$ with $t=gw$, $w$ in $\L_{T}^{1}$.
		
		\item In the $n$-gram example the only sets that appear are the graded pieces. We can however form more interesting examples involving more complex sets. For instance, given a word $w$, we can consider the mapping
		\begin{equation}
			P_{w}\colon\{w\}\to\pi^{-1}_{2k+1}(w)
		\end{equation}
		with the probabilities from equation \eqref{eq:pw2v}. Then, by aggregating over all the words in $T$, we get a matrix
		\begin{equation}\label{eq:fibers}
			\L_{T}^{1}\to\prod_{w}\pi^{-1}_{2k+1}(w)
		\end{equation}
	\end{enumerate}
\end{exa}

Furthermore, in $\P_{T}$ we have more structure that we did not have in the previous categories: a tensor product. Indeed, given two sets $X, Y$ in $\P_{T}$ we can construct its tensor product $X \otimes Y$. Since the objects of $\P_{T}$ have to relate to the text $T$, we define the tensor product of $X Y$ in $\P_{T}$, denoted $X \otimes Y$, as the set of expressions of $T$, $g\otimes t$, with $g \in X $ and $t \in Y$. That is, an element of $X \otimes Y$ starts with an expression of $X$ and ends with an expression of $Y$, with nothing in between. Notice that this product is not symmetric. Indeed, there might not be any expressions in $T$ that start with elements of $X$ and end with elements of $Y$. In that case, we write $X\otimes Y \cong \L_{T}^{0}$. Let's illustrate this with some examples.

\begin{exa}
	\begin{enumerate}
		\item There are some special cases of this product being symmetric. In particular, if we consider the graded pieces
		
		\begin{equation}
			\L_{T}^{k}\otimes\L_{T}^{m}\cong \L_{T}^{m}\otimes\L_{T}^{k} \cong \L_{T}^{k+m}.
		\end{equation}
		
		This is evidenced by the fact that we can divide any expression of length $k+m$ in two ways: the first $k$ words and the last $m$ words or the first $m$ words and the last $k$ words.
		
		\item On the other hand, if $X,Y \subset \L_{T}^{1}$ such that there are no expressions of length $2$ starting with a word in $X$ and ending a word in $Y$ one would have $X\otimes Y= \L_{T}^{0}$. That is, this product would be the $0$-th graded piece if the combinations of words in $X$ and $Y$ were not in the text $T$.
		
		\item Finally, it is worth noticing that for every $X$ in $\P_{T}$ we have that $X\otimes \L_{T}^{0}\cong X\cong \L_{T}^{0}\otimes X$. Thus, the $0$-th graded piece acts as the unit object for the product.
	\end{enumerate}
\end{exa}

\begin{re}
	The second case where the tensor product of two objects yields the unit object is reminiscent of the tensor product of abelian groups. For instance, in the case of abelian groups, we have $\Q\otimes \Z/n\Z\cong 0$.
\end{re}

Now we turn our attention to the morphisms in $\P_{T}$. Specifically, how to compose them: in the case of row-stochastic matrices there is no problem since the product of two row-stochastic matrices is row-stochastic. The problem arises when considering probabilistic matrices: the product of two matrices filled with ones is not a probabilistic matrix. To better understand how to define a composition involving probabilistic matrices we are going to look first at endomorphisms in $\P_{T}$.

Given a set $X$ of $\P_{T}$ an \textbf{endomorphism} is a morphism $f\colon X\to X$. The problem is that we cannot translate morphisms in $\L_{T}$ to stochastic matrices. Indeed, every entry of the matrix is non-negative and the diagonal is filled with ones. This means that to have a good notion of a category expressing the conditional probabilities of equations \eqref{eq:def} and \eqref{eq:endoL} we need to allow more general matrices than the stochastic ones. Not only that, we need to consider a slightly modified version of the matrix product for the composition to work.

To include those matrices we need to take a closer look at the construction of endomorphisms in $\P_{T}$. By equation \eqref{eq:endoL}, morphisms among elements of the same graded piece en $\L_{T}$ encode the probability of finding an expression instead of the other (given a text $T$). This was achieved by computing an infimum among all elements in $\L_{T}$ which was necessary to satisfy the composability condition in the enriched category $\L_{T}$. In contrast, in the not-yet-defined category $\P_{T}$ we do not have that restriction and we can choose different sets of expressions to compute the infimum of equation \eqref{eq:endoL}. This means that we can have an endomorphism $P = (p_{v,w})$ of $\L_{T}^{1}$ (as an object in $\P_{T}$) where the probabilities are computed with the formula:

\begin{equation}
	p(v\| w) = \inf_{g\in \L_{T}^{k}}\left\{\frac{p(g| v)}{p(g| w)}	\right\}.
\end{equation}

By confining the infimum to a graded component, we narrow down the space in which to seek similarities among words. This may be more efficient computationally but it gives a biased view of the meaning of the words since we are just looking for similarities in a small piece of all the categories we have at our disposition. Thus, an endomorphism in $\P_{T}$ encodes the similarity between words given a certain context which is the set we take the infimum from. 

This still leaves us with the problem of composing those morphisms. Indeed, since now we have two distinct classes of morphism, stochastic matrices that represent conditional probabilities and probabilistic matrices that represent similarities, the composition rule is not as straightforward as in the product of stochastic matrices. The key aspect is to define a composition rule that makes conceptual sense as well as being mathematically correct. To achieve it we need to satisfy the composition rules of a category meaning that composing by the identity morphism yields the same matrix and that the composition of two matrices, stochastic or probabilistic, should yield another stochastic or probabilistic matrix. Once we have a mathematically robust definition, we need to interpret it to make sure it fits the framework.

It helps to consider the situation as a game. The game consists of a player giving expressions of the text $T$ given certain information. If the information comes from the same graded piece, the player has to select expressions with similar meanings. If the information comes from a different graded piece, the player has to select possible extensions or retractions. Hence, the composition of two morphisms can be thought of as a game played in two rounds one per morphism.

With all of that in mind, we can define the composition of morphisms. Given two composable morphisms $S$ and $P$, if they are both stochastic matrices, the composition is usual matrix multiplication. This makes sense since to give the possible extensions of extensions we need to account for all of the possible ways to extend each expression of $T$. This results in each entry of the resulting matrix being a sum of probabilities of extensions.

On the other hand, if one of the matrices is probabilistic the product is not quite the usual matrix multiplication: each of the entries is going to be divided by the ceiling function of the sum of the transition probabilities. Notice that in the case of a stochastic matrix (since the sum of each column is $1$) that number is one and we recover the usual matrix multiplication.

After these definitions, it is worth working out some examples to understand the interpretation.

\begin{exa}
	\begin{enumerate}
		\item The morphism of equation \eqref{eq:fibers} is a diagonal of ones since this morphism is equivalent to choosing the middle word of a $2k+1$ expression. Since there is no uncertainty, just picking the middle word, the entries are either zero or one.
		
		\item  In example \ref{exa:probs} the morphisms can be aggregated into a matrix (a morphism) $P^{-}\colon \L_{T}^{n}\to \L_{T}^{n}\otimes \L_{T}^{1}$ indexed by the elements of  $\L_{T}^{n}$ and $\L_{T}^{n+1}$ with $P^{-}(g^{-},g^{-}w) = p(g^{-}w|g^{-})$.
	\end{enumerate}
\end{exa}

The last thing to take care of in order before we (rigorously) define the category $\P_{T}$ is to analyze how the tensor product interacts with the morphisms. If we have two morphisms $p\colon A\to B$ and $q\colon X\to Y$ we can form a new morphism $p\otimes q\colon S= A\otimes X\to B\otimes Y = T$ such that for every $t= b\otimes y \in T$ and every $s= a\otimes x \in S$ we have

\begin{equation}\label{eq:tensormorph}
	p\otimes q(t|s)= p(b|a)q(y|x).
\end{equation}

\begin{figure}
	\centering
	\includegraphics[scale=0.4]{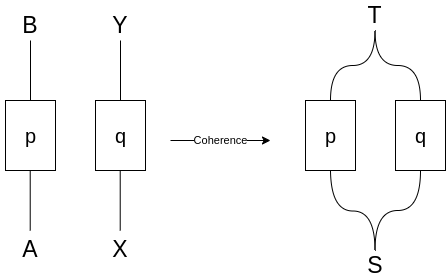}
	\caption{Diagram of the tensor product of morphisms.}
	\label{fig:tensor_morphisms}
\end{figure}

The diagrammatic form of equation \eqref{eq:tensormorph} can be found in Figure \ref{fig:tensor_morphisms} where we interpret the tensor morphism as a king of coherence condition between spaces and morphisms. Notice that to construct the tensor product of $A$ and $X$ we are implicitly using the Cartesian product $A\times X$ (which corresponds to the left-hand side of \ref{fig:tensor_morphisms}). These two products and their morphisms have very different meanings. Indeed if we have objects $A, X, Y$ in $\P_{T}$ and morphisms $p\colon A\otimes X\to Y$ and $f\colon A\times X \to Y$ we have that $p(t|a\otimes x)$ is the conditional probability of finding $t$ given the expression $a\otimes x =ax$ whereas $f(t|a, x)$ expresses the probability of finding $t$ given the expressions $a$ and $x$ (it is a conditional probability conditioned to two events). With all those preparations we can finally define the category $\P_{T}$.

\begin{definition}
	We establish the category $\mathcal{P}_{T}$, where the objects consist of sets of expressions within the text $T$ and their Cartesian products. The morphisms within this category represent either probabilistic matrices or row-stochastic matrices. Additionally, $\mathcal{P}_{T}$ functions as a monoidal category $(\mathcal{P}_{T}, \otimes, \mathcal{L}_{T}^{0})$, featuring the defined tensor product $\otimes$ and the unit object being the $0$-th graded piece.
\end{definition}

\begin{re}
	
	We can have a different monoidal structure on $\P_{T}$ making it symmetric monoidal. Indeed, we can  consider the product of two objects $X$ and $Y$ to be $ X\boxtimes Y$ the Cartesian product of the sets and the product of morphisms $f\boxtimes g$ the Kronecker product of matrices. In that case, the identity of the product is again the $0$-th graded piece\footnote{We have opted to write $\C_{T}^{0}$ instead of $\L_{T}^{0}$ because we only want to consider stochastic matrices coming from morphisms in $\C_{T}$ for the construction of the Markov category.} $\C_{T}^{0}$. Notice that $ X\boxtimes Y\cong Y\boxtimes X$ makes it indeed symmetric.
	This makes the subcategory of $(\P_{T},\boxtimes \C_{T})$ of sets of expressions of $T$ and their Cartesian product, row-stochastic matrices as morphisms a Markov category in the sense of \cite{fritz2023representable} and \cite{perrone2023markov}.
\end{re}

\begin{theorem}\label{thm:MonoidalCat}
	The category $(\P_{T},\otimes, \L_{T}^{0})$ is indeed a monoidal category.
\end{theorem}

The proof of the theorem can be found in Appendix \ref{sec:AppendixA}, Theorem \ref{thm:MonoidalCatA}. This category can be understood as the category of Alexandrov cones starting at $\L_{T}^{0}$ where each cone displays different relations among its elements. As seen in figure \ref{fig:PTCones} we have several cones emanating from a single point (The $0$-th graded piece). We have shown two morphisms $f_1$ that goes from som $X$ in $\P_{T}$ to $B$ and an endomorphism of $\L_{T}^{1}$, that is $S\colon \L_{T}^{1}\to \L_{T}^{1}$.

The reason we have represented the arrow $S$ going from one cone to another is that each cone can be interpreted as a different $\L_{T}$ with different conditional probabilities and we can pass from one to another via endomorphisms. Indeed, having conditional probabilities on a matrix implies the matrix is stochastic. Assume this matrix to be $p\colon X\subseteq \L_{T}^{1}\to Y$. Then the composition of $p$ with $S$ restricted to $X$ yields $p\circ S_{|X} \colon X\to Y$ another morphism from $X$ to $Y$ but with different probabilities. Thus, we have effectively modified the space to a new space that still looks like a cone but with a different internal structure.

Moreover, this implies that given a system of stochastic matrices $\{p_k\}$ between graded pieces we can create a cone. If we add an endomorphism of $S\colon \L_{T}^{1}\to \L_{T}^{1}$ to the system $\{S,p_k\}$, by composition, we can create two cones, and so on. Hence, the cones are determined by the endomorphisms of $\L_{T}^{1}$.

\begin{figure}
	\centering
	\includegraphics[scale=0.3]{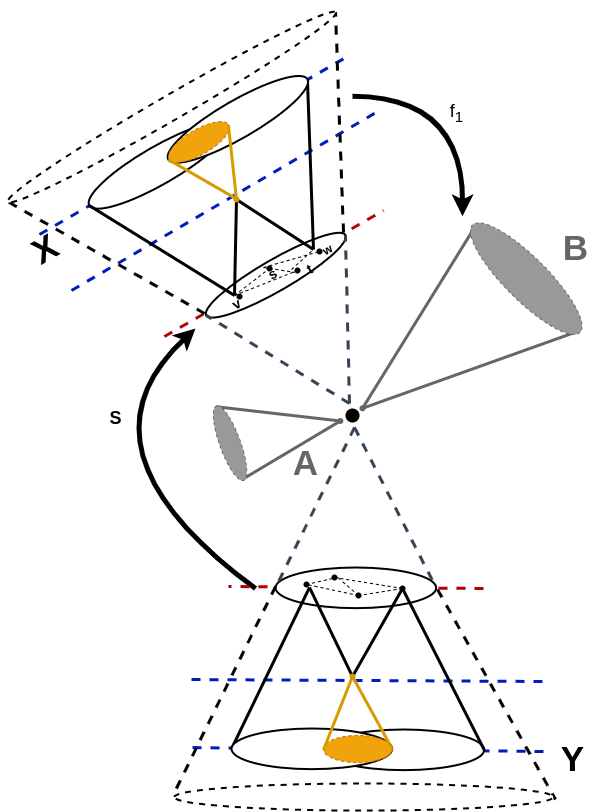}
	\caption{The category $\P_{T}$ as Alexandrov Cones with the action of its morphisms.}
	\label{fig:PTCones}
\end{figure}

\begin{exa}[Semantic Telephone]\label{exa:st}
	Given an endomorphism $f$ of $\L_{T}^{1}$ we can compose it with itself $k$ times to create a sequence of probable words after $k$ iterations. This means that we can interpret this composition as a game (again) where we ask what is the most probable word given a certain word $k$ times and then find the result. This yields a certain dynamic behavior in $\L_{T}^{1}$.
	
	Similarly, if we start with an expression $g$ in $\L_{T}^{n}$ and we have $k$ morphisms $f_{1},\dots, f_{k}$, the result of $(f_k\circ\cdots\circ f_1)(g)$ is a vector with probabilities. The expression $t$ in $\L_{T}^{n}$ with the highest probability can be understood as the result of a game of telephone where each player passes the information (the expression $g$) to the next which interprets it with probability matrix $f_i$. The information is going to get distorted and it is unlikely that we recover $g$ at the end unless the $f_i$'s we close to the identity.
	
	This uncertainty comes from the meaning of expressions based on a text, that is from semantic similarity between the expressions, whereas in the regular game of telephone, the uncertainty comes from phonetic similarity. That is why we have named this game the ``semantic telephone". It encodes the noise a channel can have when transmitting information.
\end{exa}

Figure \ref{fig:PTCones} helps us visualize the game of semantic telephone as jumping from elements of $\L_{T}^{n}$ of one cone to another. The number of endomorphisms present is the number of cones we jump to which is also the number of players in the game. These endomorphisms represent the certainty of finding words in similar contexts. This means that we can think of them as semantic similarities between words yielding a certain arrangement of $\L_{T}$. Furthermore, since having an endomorphism of $\L_{T}^{1}$ (which is a matrix) is the same as giving a list of points and their relations\footnote{This is the same phenomenon that occurs in graph theory, where having the adjacency matrix is equivalent to having the weighted graph.} we have the following definition.

\begin{definition}
	Given a text $T$ and the category $\P_{T}$ we define a \textbf{semantic space} to be an endomorphism\footnote{To base the definition of a semantic space makes sense since, as discussed earlier, we can reconstruct the $\L_{T}$ from the stochastic matrices and an endomorphism.}  of $\L_{T}^{1}$in $\P_{T}$.
	
\end{definition}

This definition may seem strange but since classically semantic spaces are representations of natural language that are capable of capturing meaning and endomorphisms of $\L_{T}^{1}$ to achieve that goal it makes sense to call them thusly. Furthermore, since having the matrix is equivalent to having the words and their relations we do have all the information needed. 

\begin{re}
	Note that we treat endomorphisms in $ \P_{T} $ as distinct entities in two regards. Firstly, they provide insights into the structure of the space where they are defined. Secondly, we conceptualize them as arrows between objects within different cones, where the elements remain consistent but their relationships undergo changes. This dual perspective reflects an aspect of the Yoneda embedding (Lemma \ref{lem:yoneda}), where each object $ X $ is akin to its functor $ y^{X} $. In essence, viewing a semantic space as a matrix representing word relationships is equivalent to regarding it as an arrow originating from the space with the identity matrix. 
\end{re}

There are other definitions of semantic spaces, for instance in \cite{hashimoto2016word} and in many other works, where the authors define them to be vector spaces over concepts where Euclidean distances between points indicate semantic similarities. However, this definition imposes two properties that are not intrinsic to language: dimension and composition. 

Expanding on this, the necessity for it to be a vector space implies the requirement for a dimension to be chosen. The most intuitive selection would align with the size of the dictionary, facilitating one-hot encoding. However, it's worth noting that a lower-dimensional space might effectively capture semantic similarities, while conversely, a higher-dimensional space may not precisely reflect the meanings of words. Consequently, dimensionality isn't inherently inherent to a semantic space.

The second characteristic introduced by this definition is linearity and composability. While undoubtedly a valuable trait, it doesn't entirely align with the natural essence of language. Although words can indeed be combined akin to the summation of vectors, the operation in the vector domain is commutative, unlike the non-commutative nature of language. Additionally, not every combination of $n$ words forms a meaningful sentence, whereas any set of $n$ vectors can be summed. While vector spaces offer a strong representation of linguistic information, they inherently introduce additional parameters and properties not intrinsic to language itself. This delineation underscores our preference for a more abstract yet still representable definition.

In the next section, we explore how to incorporate the ``vector space" view of semantic spaces into the categorical framework through word embeddings.

\section{Word Embeddings}\label{sec:WE}

In this section, we establish word embeddings derived from semantic spaces. Given that we'll address both words within $ T $ (or $ \L_{T}^{1} $) and word vectors in $ \R^{n} $, we'll establish specific notation. When referencing words in $ T $, we'll employ the notation $ w_{i} $, while for embedded word vectors, we'll denote them as $ v_{i} $. To define word embeddings, we'll begin by introducing a particular type of object known as configurations.

\begin{definition}
	\begin{enumerate}
		\item A \textbf{configuration} $(X, C_{X})$ consists of a (pseudo)metric space $X$ and a collection of points $C_{X}=\{p_i\in X\colon i=1,\dots,n\}$.
		
		\item A \textbf{morphism of configurations} consists of a continuous map
		\begin{equation}
			f\colon (X,C_{X})\to(Y,C_{Y})
		\end{equation}
		such that $f(C_X)\subseteq C_{Y}$.
		
		\item These form a category denoted $\Conf$.
	\end{enumerate}
\end{definition}

\begin{re}
Notice that this definition yields the concept of \textbf{equivalent configurations}. Two configurations $ (X,C_{X}),(Y,C_{Y}) $ are equivalent if there exists a continuous conformal map $ f\colon X\to Y $ of (pseudo)metric spaces such that $ f(C_{X}) = C_{Y} $. In particular, in the case of linear configurations, where both $ X $ and $ Y $ are vector spaces equipped with the usual vector distance, there is a stronger notion where the mapping is not only continuous but also linear.
\end{re}

The idea is that to have a configuration of $n$ points we have to specify two things: the space $X$ (which usually is some vector space of dimension $d\leq n$) and a set of points $C_{X}$ of $X$. Since the set $C_{X}$ is a subset of $X$, we have a matrix $M_{C}$ of the distances between any two points. With that in mind, we can define the category of word embeddings.

\begin{definition}\label{def:wordembeddings}
	The category of word embeddings of a text $T$, $\Emb$ is the category whose objects are maps:
	\begin{equation}
		\P_{T}(\L^{1}_{T},\L^{1}_{T})\to\Conf
	\end{equation}
	
	and morphisms are commutative squares
	
	\begin{equation}\label{eq:defmorphWE}
		\begin{tikzcd}
			P \arrow[d,"h"] \arrow[r]& (X,C_{X})\arrow[d,"r"] \\
			P'\arrow[r] & (Y,C_{Y})
		\end{tikzcd}
	\end{equation}
\end{definition}

In the equation \eqref{eq:defmorphWE}, the arrow $h$ represents a change of probabilities or dissimilarities and the arrow $r$ represents a change in the configuration. Of special interest is the case when $(X, C_{X})=(X, C'_{X})$ with $|C_{X}|=|C'_{X}|$, that is when the space and the number of points is the same but the points themselves have changed. In that case, equation \eqref{eq:defmorphWE} becomes

\begin{equation}
	\begin{tikzcd}
		P \arrow[d,"h"] \arrow[r]& (X,C_{X})\arrow[d,"r"] \\
		P'\arrow[r] & (X,C'_{X})
	\end{tikzcd}
\end{equation}

This equation implies that a change in probabilities corresponds to a change in the configuration. Of course, it would be interesting to know exactly how this change in probabilities would change the configuration. For that, we are going to extend the notion of divergence to the category of word embeddings. The problem is that there are at least two possible ways to define a divergence on $\Emb$:
\begin{itemize}
	\item We can assign to each object $P\to (X,C_{X})$ of $\Emb$ a divergence between the similarities of $P$ and the distance (dissimilarities) of $C_{X}$. 
	
	\item A divergence comparing the error made by two different objects of $\Emb$. This is the closest to the divergence as defined in \citep{perrone2023markov} since it is a divergence on the set of morphisms between two objects.
\end{itemize}

Since we want to assign a divergence to both objects and morphisms a decoration in the sense of  \cite{fong2015decorated,baez2021structured} would be appropriate. On top of that, we want the assignment to be functorial in some way. Hence, we want to consider a functor from the category of embeddings to another category as in \cite{AFH}. 

\begin{definition}\label{def:deco}
	Let $\mathcal{A}$ a category and $\mathcal{M}$ a monoidal category. An $F$\textbf{-decoration} of $\mathcal{A}$ is a functor $F\colon \mathcal{A}\to \mathcal{M}$.  
\end{definition}

In our case, since we want the divergence to express the error of embedding the monoidal category $\mathcal{M}$ of the definition is going to be $([0,+\infty],\geq, +)$. This yields the following definition.

\begin{definition}
	A \textbf{divergence} on $\Emb$ is a decoration $D\colon \Emb \to [0,+\infty]$. That is, it is a functor $D\colon \Emb \to [0,+\infty]$ such that for every object $\mathcal{E}\colon P\to (X,C_{X})$ we have $D(\mathcal{E}\|\mathcal{E})=0$.
\end{definition}

\begin{exa}
	One of the most known divergences is the Kullback-Leibler divergence between two probability measures $p,q$ on a set $X$:
	
	\begin{equation}
		D_{KL}(p\| q)=\sum_{x\in X}p(x)\ln\left(\frac{p(x)}{q(x)}\right).
	\end{equation}
	
	In the setting of the category $\Emb$ notice that this can be obtained for objects in $\Emb$ $f\colon P\to (X,C_{X})$ using the matrix of distances $D_{X}=(d_{ij})$ of point in $C_{X}$ via the formula:
	
	\begin{equation}
		D_{KL}(f)=\sum_{i<j}p_{ij}|\ln(p_{ij})+d_{ij})|.
	\end{equation}

	Given two embeddings $f\colon P\to (X,C_{X})$ and $g\colon Q\to(Y,C_{Y})$, their divergence is just $D_{KL}(f\| g)=|D_{KL}(f)-D_{KL}(g)|$. The fact that this satisfies the composability condition of a functor is given by the triangle inequality\footnote{This effectively turns $\Emb$ into a Lawvere metric space. See Section 3 in \url{https://ncatlab.org/nlab/show/metric+space}.}.
\end{exa}

Hence, a divergence on objects measures the difference between the information captured by $P$ and the information represented by $\C_{X}$. On morphisms, it measures how different those embeddings are. Indeed, the exact points of an embedding are not relevant, only their relations since we can just apply an isometry of the space and recover an embedding that has different positions but the same ``logical structure". 

This naturally poses the question: when are two embeddings equivalent? One would expect two embeddings to be equivalent when, in the presence of the same information, they produce the equivalent configuration. Now the question is: what are equivalent configurations? This is a simpler question since configurations can be regarded as complete weighted graphs. The last piece of the puzzle is how to measure the difference in configurations. The answer is quite clear: the divergence. These considerations yield the following definition.

\begin{definition}
	We say that two embeddings $\mathcal{E}\colon P\to (X,C_{X})$ and $\mathcal{E}^{\prime}\colon P\to(X,C^{\prime}_{X})$ are \textbf{equivalent with respect to the divergence $D$} when $D(\mathcal{E}\| \mathcal{E}^{\prime})=0$.
\end{definition}

\begin{re}
Notice how the equivalence of two morphisms is a relative notion: it depends on the divergence used to measure the differences. 

\end{re}

The divergence is just the mathematical name we have given this difference because it comes from information theory. However, in a machine learning context, the divergence takes the name of \textit{error or loss function} of the algorithm. This implies that being equivalent as word embeddings is tantamount to minimizing the same error function when in the presence of the same information. With this we can find equivalent embeddings to the GloVe embedding (see \cite{pennington2014glove}) and the Word2Vec embedding (see \cite{Mikolov2013EfficientEO}).

\begin{exa}\label{exa:gv-w2v}
	\begin{enumerate}
		\item In \cite{pennington2014glove}, they commence with a co-occurrence matrix $X$, where each row is normalized to produce a row-stochastic matrix $P$. This serves as the input matrix for the embedding process. Subsequently, these probabilities undergo a conversion into distances for the embedding process. Consequently, the matrix $P$ is transformed into a matrix $d_{GV}$. However, this presents a challenge: $d_{GV}$ represents distances in a pseudo-metric space, resulting in a symmetric matrix.
		
		To solve this problem, the authors consider the equation
		
		\begin{equation}
			F(v_i-v_j,\tilde{v_k})=\frac{P_{ik}}{P_{ij}}
		\end{equation}
		with the $w_l$ the vectors of the embedding, thus transforming the ratio of probabilities in the difference of distances. The only continuous function that behaves this way is the exponential function $\exp$. In the end, the formula resolves into:
		
		\begin{equation}\label{eq:GloVeCondition}
			F(v_i^{T}\tilde{v_k})=P_{ik}
		\end{equation}
		
		where $v_i$ is the vector of the target word, $\tilde{v_k}$ is the vector of the context word and $F$ transforms the linear information (so it is a linear map) into probability. Equation \eqref{eq:GloVeCondition} has the formal problem that it is not symmetric since $P_{ik}$ and $P_{ki}$ might not coincide. To solve this issue the authors of \cite{pennington2014glove} add extra terms

		\begin{equation}\label{eq:GloVeConditionbias}
			v_i^{T}\tilde{v_k}=\log(P_{ik}) + a_{i}+ b_{k}
		\end{equation}
		
		to ensure that the role of target and context words can be reversed we get $v_i^{T}\tilde{v_k}=v_k^{T}\tilde{v_i}$ which implies

		\begin{align}
			2 v_i^{T}\tilde{v_k}&=v_i^{T}\tilde{v_k} + v_k^{T}\tilde{v_i}\\
			&= \log(P_{ik})+\log(P_{ki}) + a_{i}+ b_{k} + a_{k}+ b_{i}
		\end{align}
		
		Thus
		\begin{align}
			G_{ik}&=v_i^{T}\tilde{v_k}=\frac{\log(P_{ik})+\log(P_{ki})}{2} + \frac{1}{2}( a_{i}+ b_{k} + a_{k}+ b_{i})\\
			&=v_k^{T}\tilde{v_i}=\frac{\log(P_{ik})+\log(P_{ki})}{2} + \frac{1}{2}( a_{i}+ b_{k} + a_{k}+ b_{i})\\
			&=G_{ki} .
		\end{align}
	
		The matrix $G$ is known as the Gramm matrix. From this we can compute the matrix $(d_{GL})^{2}_{ik}=G_{ii}-2G_{ik}+G_{kk}$.  This yields the distance formula:
		
		\begin{equation}\label{eq:GloVedist}
			(d_{GV})_{ik}=\sqrt{\log\left(\frac{P_{ii}\cdot P_{kk}}{P_{ik}\cdot P_{ki}}\right)}=(d_{GV})_{ki}
		\end{equation}

		Given a configuration $(X,C_{X})$ with $D_{X}$ the divergence of on iteration of the GloVe embedding would be:
		
		\begin{equation}\label{eq:DivGL}
			D_{GV}(P\to(X,C_{X}))=\sum_{i, j}W_{ij}((d_{GL})_{ij}-(d_{X})_{ij})^{2}
		\end{equation}
		with $W_{ij}$ a weight attached to the co-occurrences. The final piece is to specify how the GloVe embedding computes the probabilities $P_{ij}$. Following \cite[Section 3]{pennington2014glove} we have
		
		\begin{equation}\label{eq:EndoGV}
			P_{ij}=p(w_{j}\| w_{i})=\frac{X_{ij}}{X_{i}}=\frac{C_{|\pi^{-1}(w_i)}(w_j)}{\sum_{k}C_{|\pi^{-1}(w_i)}(w_k)}
		\end{equation}		
		
		with$C_{|\pi^{-1}(w_i)}(w_j)$ tallies the co-occurrences of $w_j$ with respect to $w_i$ (which is why we compute the occurrences in the fibers $\pi^{-1}(w_i	)$.
		
		\item The Word2Vec embedding, found in \cite{Mikolov2013EfficientEO}, has two variants: the CBOW and Skip-gram variants. They are dual to each other. The CBOW variant has the form of a map $Q\to (\R^{n},C_{n})$, for some $n\in \N$, with the matrix $Q$ obtained from equations \eqref{eq:pw2v} and \eqref{eq:fibers}. The semantic space $Q$ is then recovered using the following formula:
		
		\begin{equation}\label{eq:EndoW2V}
			Q_{ij}=q(w_j\| w_i)=\frac{C_{|\pi_{-1}(w_i)}(w_j)}{\sum_{w\leq t\in\pi^{-1}(w_i)}C(t)}
		\end{equation}
		
		for $w\in \mathcal{L}_{T}^{1}$.

		The Skip-gram model is the dual of the CBOW embedding given by the matrix of equation \eqref{eq:fibers} (see also \cite[Section 2]{levy2014neural}). On the embedding space side of things we have that the divergence is computed using the softmax function $\sigma\colon \R^{V}\to [0,1]^{V}$:
		$$\sigma(x)=\left(\frac{\exp(v_{1})}{\sum_{i=1}^{n}\exp(v_{i})},\dots,\frac{\exp(v_{n})}{\sum_{i=1}^{n}\exp(v_{i})}\right)$$
		
		then, the divergence is the negative log-likelihood for a given target word $w_{t}$ and context word $w_c$
		
		\begin{equation}\label{eq:DivW2V}
			D_{W2V}(Q\to (\R^{n},C_{n}))=-\sum_{w_t, w_c}\log\left( \frac{\exp(v_t^{T}v_c)}{\sum_{i=1}^{V}\exp(v_{t}^{T}v_i)}\right)
		\end{equation}
		where $V=|\L_{T}^{1}|$.
		
		\item Let $S$ be a symmetric matrix representing an endomorphism of $\L_{T}^{1}$. The MDS embedding $S\to (\R^{n},C_{n})$ has a divergence called \textbf{stress}
		
		\begin{equation}\label{eq:DivMDS}
			D_{MDS}(v_1,\dots,v_n)=\sqrt{\sum_{i\neq j}(d_{ij}-\|v_i-v_j\|)^{2}}.
		\end{equation}
		The embedding is the result that makes this stress minimal.	There are several possibilities for $ d_{ij} $, for instance we can have  $d_{ij}=-\log (S_{ij})$.
	\end{enumerate}
\end{exa}

Before we start the statement of the main theorem we need to specify what are the divergences of the GloVe and Word2Vec embeddings. Since their algorithms are iterative and our definition considers an embedding just having a divergence (that is a numerical value attached to it) we need to specify to which of its iterations we are referring to. Notice that an iterative algorithm of word embeddings can be understood as a sequence $\{\mathcal{E}_{n}\colon P\to(X, C^{n}_{X})\}$ of embeddings yielding a sequence of divergences $\{D_{n}\}$. If the algorithms converge, these sequences are going to converge to $\mathcal{E}^{*}$ and $D^{*}$ respectively. Thus, when we talk about the GloVe, Word2Vec, or MDS embeddings, we are referring to these convergence embeddings. With this in mind, we say that \textbf{two word embedding (iterative) algorithms $\{\mathcal{E}_{n}\colon P\to(X,C^{n}_{X})\}$ and $\{\mathcal{E}_{n}^{\prime}\colon P\to(X,C^{\prime n}_{X})\}$ are equivalent} when their sequences converge to the same word embedding $\{\mathcal{E}^{*}\colon P\to(X,C^{*}_{X})\}$with divergence $D^{*}$.

As a final note before stating the lemma, by example \ref{exa:gv-w2v} and the preceding paragraph, when we have different classes of embeddings the divergence on morphisms can be taken to be the absolute value (to make it greater than $ 0 $) of a polynomial in the two types of divergences (see Remark \ref{re:diffdiv}). Thus, we are somewhat making the divergence on morphism an element of $ \Poly $ (see \cite{niu2023polynomial}).
\
\begin{lemma}\label{lem:convergence}
	The word embeddings $\{\mathcal{E}\colon P\to(X,C_{X})\}$ and $\{\mathcal{E}^{\prime}\colon P\to(X,C_{X})\}$ are equivalent if there is a sequence of divergences between them that convergences to $0$. In particular, if the sequences of divergences are decreasing, checking the equivalence of the embeddings is tantamount to verifying the conditions yielding the infimum for both sequences are the same.
\end{lemma}

The proof can be found in Appendix \ref{sec:AppendixA}, Lemma \ref{lem:convergenceA}.

\begin{re}\label{re:diffdiv}		
	\begin{enumerate}
		\item This definition of equivalent algorithms is reminiscent to the definition of equivalent sequences in the construction of the real numbers from Cauchy sequences in the rationals. This makes sense since each divergence turns $\Emb$ in a Lawvere (metric) space where we can compute convergence. 
		
		\item Notice how in lemma \ref{lem:convergence} we have not specified the divergence. This is because the divergence that converges to $0$ will depend on the divergences of the embeddings $\mathcal{E}$ and $\mathcal{E}^{\prime}$. For instance, the divergence of the MDS embedding can theoretically be $0$ but the divergence of the Word2Vec algorithm is the Kullback-Leibler divergence which can only be $0$ if the probabilities match and the entropy of one (and hence the other) probability distributions is $0$. Hence,by setting $D(W2V \| MDS)=|D_{W2V}-(D_{MDS}+H(Q))|$ we eventually get $D(W2V \| MDS)=0$. But again, it is dependent on the specific embeddings.
		
		\item In particular, since the divergence $ D $ varies depending on the embeddings, the equivalence is made in different Lawvere (Banach) spaces. 
		
		\item  Finally, it is worth noting that this lemma gives a precise method to show equivalence of any pair of word embeddings not justbthe ones we focus on the manuscript.
		
	\end{enumerate}
\end{re}

With that in mind, we can finally state the main theorem.

\begin{theorem}\label{thm:main}
	The GloVe and Word2Vec neural network embeddings are equivalent to metric MDS embeddings on vector spaces.
\end{theorem}

The proof can be found in Appendix \ref{sec:AppendixA}, Theorem \ref{thm:mainA}.

\begin{re}[Weights of the loss functions]
	Notice how both loss functions (for the GloVe algorithm and the Word2Vec algorithm) are weighted by $f(X_{ij})$. This does not modify the minimums of the loss functions as seen in the proof of Theorem \ref{thm:main}. Indeed, those weights are fixed and not subject to modification by the optimization process. They are useful, however to the optimization process. These weights are passed via gradient descent (in the back-propagation step) to speed up the convergence of the algorithm. 
	
\end{re}

Guided by example \ref{exa:gv-w2v} given a a semantic space, that is an endomorphism $P\in \P_{T}(\L_{T}^{1},\L_{T}^{1})$, we can construct a similarity matrix $S$ satisfying:

\begin{itemize}
	\item $S$ is a symmetric matrix. This means $S^{T}=S$.
	\item $S_{ii}=1$ for all $i$. This means that the diagonal entries are equal to $1$.
\end{itemize}

\begin{exa}
	\begin{enumerate}
		\item Example \ref{exa:gv-w2v} provides an illustration of such matrices: given a semantic space $P\in \P_{T}(\L_{T}^{1},\L_{T}^{1})$, we can construct the symmetric matrix with entries.
		
		\begin{equation}
			S_{ik}=\frac{P_{ii}\cdot P_{kk}}{P_{ik}\cdot P_{ki}}.
		\end{equation}
		The matrix $S$ is indeed symmetric with ones in the diagonal.
		
		\item Another example is
		
		\begin{equation}
			T_{ij}=\frac{P_{ij}+P_{ji}}{P_{ii}+P_{jj}} 
		\end{equation}
	\end{enumerate}

\end{exa}

This helps us to define the bias in a semantic space. The idea of bias is that there is more correlation between a pair of words than another pair than it should be. For instance, if there is more correlation between the words ``man" and ``doctor" than between the words `woman" and ``doctor" we can say that the concept of ``doctor" has a gender bias. This, in turn, gets passed down to the embedding via distances. In the case of the word ``doctor", its word-vector would be closer to the word-vector of ``man" than the word-vector of ``woman". This yields the following definition:

\begin{definition}[Bias of a pair with respect of a term]\label{def:bias}
	Given a semantic space, that is an endomorphism $P\in \P_{T}(\L_{T}^{1},\L_{T}^{1})$, we define the \textbf{bias of $w_{i}$ with respect to the pair $(w_{k},w_{j})$ and $S$} by the quotient:
	
	\begin{equation}
		b_{i}(k,j)=\frac{S_{ik}}{S_{ij}}
	\end{equation}
	
	with $S$ being a similarity matrix derived from the space $P$.
\end{definition}

\begin{re}
	By definition \ref{def:bias}, if a term is \textbf{unbiased} with respect to a pair and a similarity matrix, the bias is $1$. This essentially means that correlation between the word and each term of the pair is the same which makes sense conceptually.
\end{re}

Since there are several ways to get a similarity matrix from a semantic space, the bias is relative to $S$, which in turn makes it model-dependent. This means that to remove the bias present in the word-vector side of things, one needs to understand how the embeddings work: what is the semantic space examined, and how it is used in the embedding to produce distances?

\begin{cor}[Bias Reduction]\label{cor:biasred}
	Let  $b_{i}(k,j)$ the bias of $w_{i}$ with respect to the pair $(w_{k},w_{j})$ and $S$. To remove this bias for the GloVe and Word2Vec embedding we need to equalize the quotients:

	\begin{equation}\label{eq:unbias}
		\frac{p_{kk}}{p_{ik}\cdot p_{ki}} = \frac{p_{jj}}{p_{ij}\cdot p_{ji}}
	\end{equation}
	with $p$ representing the probabilities of example \ref{exa:gv-w2v} $p=P$ in the case of GloVe and $p=Q$ in the case of Word2Vec. This ensures that after the embedding $d_{ik}=d_{ij}$ in both cases.
	
\end{cor}

The proof can be found in Appendix \ref{sec:AppendixA}, Corollary \ref{cor:biasredA}. By this process, we can equalize the distances of the word vector $v_i$ to the pair $(v_k,v_j)$ which removes a specific bias. Unbiasing the whole embedding is a different story. In order completely remove a specific bias, say to the pair $w_k, w_j$ we would need to equalize $b_{i}(k,j)$ for all  $i$ that make sense. For instance, it is conceivable that the word ``pregnant" is more closely related to the word ``woman" than to the word ``man". Thus, the operation of unbiasing a word embedding is also context/word dependent as explained in \cite{caliskan2022gender}. 

Another consequence of Equation \eqref{eq:unbias} is that one way to remove gender bias from an embedding one would need to make every applicable cross-probability $ p_{ij} $ equal to $ p_{ik} $ and $ p_{ji} $ equal to $ p_{ki} $. This is a way to make sure that the words are gender agnostic. It is not however the only step, one also needs to adequately change $ p_{jj}$. This reflects that there is a notion of centrality linked to bias. Indeed, if we only change the cross-probabilities the quotients might still differ because one word is more relevant to its surroundings than the other.

For example, if $ p_{jj}\geq p_{kk} $, it implies that the word $ w_{j} $ is surrounded by fewer distinct words compared to $ w_{k} $, as defined by these probabilities. Consequently, $ p_{jj} $ holds more relevance within its context than $ p_{kk} $, indicating that $ w_{j} $ is less central according to the framework outlined in \citep[Section $ 3 $]{hashimoto2016word}. This observation underscores how bias is directly linked to the relative ubiquity of one word in comparison to another within a given pair. Exploring this aspect further could offer promising avenues for mitigating biases in word embeddings.

As a final remark, notice how this framework allows us to ``visualize" the two approaches to mitigate bias. If we act on the embedded side we have a change of configuration:

\begin{equation}\label{diag:debiasmetric}
	\begin{tikzcd}[row sep=small]
		& (X,C_{X}) \arrow[dd] \\
		P \arrow[ur, "\mathcal{E}"] \arrow{dr}[swap]{\mathcal{E^{\prime}}} & \\
		& (X,C^{\prime}_{X})
	\end{tikzcd}
\end{equation}

where $(X, C^{\prime}_{X})$ is the new configuration where bias has been removed. If the divergence of the embedding has not changed, and it was minimal for $ \E $ then $ D(\E^{\prime})\geq D(\E)$. This means that we have a ``worse" embedding since it has more error. This implies that one cannot achieve this embedding by an optimization process that minimizes the error.

On the other hand, if act on the semantic space side we have a diagram:

\begin{equation}\label{diag:debiasprob}
	\begin{tikzcd}[row sep=small]
		P\arrow[dd]\arrow[dr, "\mathcal{E}"]& \\
		& (X,C_{X})\\ 
		P^{\prime}\arrow{ur}[swap]{\mathcal{E^{\prime}}}& 
	\end{tikzcd}
\end{equation}

Here, the main difference is that we have two embeddings with the same configuration but different semantic spaces. Here again, we expect to have $ D(\E^{\prime}) \geq D(\E) $ assuming the divergence was minimal for $ \E $. This is again a case of the embedded points not reflecting accurately the information found in the semantic space.

Thus, in order to account for this change in the semantic space, we need to reflect it as a morphism in $ \Emb $:

\begin{equation}\label{diag:completedebias}
	\begin{tikzcd}
		P \arrow[d,"h"] \arrow[r,"\E"]& (X,C_{X})\arrow[d,"r"] \\
		P'\arrow[r,"\E^{\prime}"] & (X,C'_{X})
	\end{tikzcd}
\end{equation}

Here we have a change of embeddings, both of which can be optimal. Consequently, both divergences $ D(\E) $ and $D(\E^{\prime})$ are minimal (not necessarily equal) and there is a way to extend this divergence on objects to a divergence on morphisms such that: $ D(\E\|\E^{\prime})=0 $ reflecting the optimality of both embeddings given their semantic spaces.

\section{Conclusions and Future Work}\label{sec:CFW}

In this section, we summarize the key findings and contributions of our framework. We discuss the implications of our results, highlight areas for future research, and suggest potential directions to advance the field. 

\subsection{Conclusions}
In this work we extended the known results in the categorical approaches of \citep{CatSem} and \citep{perrone2023markov} applied to categories with semantic information and arrived at some equivalences between neural word-embedding (black-box) and MDS embeddings (white/crystal-box).

Starting with the most natural (enriched) category $\C_{T}$ we have shown how the semantic structure (of extensions and its probabilities) is intrinsically categorical. This has helped to build an intuition of how to visualize the directly apparent structure (Figure \ref{fig:ACbien}). This structure was quite useful for the first step but had one crucial limitation: we cannot compare expressions that are not directly linked to each other. This means that if expression $t$ is not an extension of expression $g$ we have no way to compare them. This was solved by the (enriched) Yoneda embedding. 

Thanks to the enriched Yoneda embedding (which turns out to be an isometry) we have constructed and extension of $\C_{T}$ denoted by $\L_{T}$. In this new category, we can compute the ``conditional probabilities" between any two expressions. In particular, it allows us to define the similarity or probability between two words. With this, we upgraded the visualization of the space to Figure \ref{fig:AlexandrovConesLT}. In this new figure we see how any two objects found in any graded piece can be compare. This has the effect of ``adding a dimension" to the figure making each graded piece a disc where we find expressions closer or further away to each other. It is the first step towards an embedding of this structure.

The structure we've established allows us to view the operation of selecting the most probable extension, given a context, as a categorical operation involving the (weighted) colimit of a specific diagram (refer to Theorem \ref{thm:colimmaxprob}). This parallels findings in \citep{pugh2023using}, where the authors derive the nearest neighbor algorithm using categorical methods.

This demonstrates that determining the most probable extension of an expression can be accomplished predictably and explainability, rather than relying solely on a black-box algorithm. The results and visualizations within the context of $\L_{T}$ are attributed to the category's inherent "rigidity." By this, we mean that while there may be a choice in how probabilities are computed, the composition rule in enriched category theory enforces strict equations for all other probabilities. As a result, variations in counting methods or smoothing techniques, such as Laplace smoothing, are not feasible within this category. To address this limitation and accommodate more general cases, we introduce the category $\P_{T}$ in Section \ref{sec:P_T}.

In Section \ref{sec:P_T} we introduced one of the main categories of the manuscript: category $\P_{T}$. This category was based on the previous category $\L_{T}$ and the Markov categories found in \cite{fritz2023representable,perrone2023markov}. The main idea was to incorporate all the results from the previous section and to allow more probabilistic structures to be considered simultaneously. This was achieved by considering the objects in the category to be sets of expressions in $T$ and morphisms probabilistic matrices between them which allowed us to model a great variety of statistical phenomena. The schematic picture is found in Figure \ref{fig:PTCones} where we can represent each probabilistic structure of $\L_{T}$ as a different cone and morphisms arrows from one piece of the same cone to another or an arrow switching cones.

Furthermore, we studied a tensor product making the category $\P_{T}$ a monoidal category enabling us to consider probabilities conditioned to the composition of expressions as opposed to probabilities conditioned to two expressions. With this we saw that the endomorphisms in this category carried important information: the transmission noise and semantic similarity (see Example \ref{exa:st}). This is shown in the representation of $\P_{T}$ by noting that an endomorphism was an arrow from one cone to another representing the change of probabilities for the same information. Thus, endomorphisms carry the differences in semantical structure.

Since the endomorphisms encode the semantical structure, we have defined in Section \ref{sec:WE} semantic spaces as endomorphisms in $\P_{T}$ of the object $\L_{T}^{1}$. The main distinction between our concept of semantic spaces and the one presented in \citep{hashimoto2016word} lies in its dimension neutrality, unlike the necessity to specify a particular dimension for the vector space in their approach. Furthermore, our definition does not mandate any additional structure beyond what is inherent in the provided text corpus $T$, as language does not exhibit commutative properties.

After establishing this framework, we introduced the category of word embeddings, where arrows denote mappings from semantic spaces to configurations of the form $(X, C_{X})$. This category was further enriched with a functorial decoration represented by a divergence, capturing embedding errors on objects and disparities in embeddings across morphisms. The culmination of this approach was Theorem \ref{thm:main}, demonstrating the equivalence between GloVe and Word2Vec embeddings (black-box) and metric MDS embeddings (white/crystal-box). This equivalence sheds light on certain embedding properties prior to their computation. Notably, the proof yields a distance formula derived from the algorithms' loss functions, independent of any ``hidden variables" or biases necessitating optimization. Particularly, we defined bias, which implies asymmetry in embedded word-vector distances, allowing for bias checks at the semantic space level before the actual embedding. This capability enables bias detection in text corpora before training, thereby avoiding their perpetuation.

As a final remark, it is worth noting that these structures solely rely on the extension structure (poset) and the statistical properties that can be extracted from corpora of texts. This implies that one could apply the same results, albeit with slight modifications, to sub-word tokens as the ones used by OpenAI\footnote{\url{https://platform.openai.com/tokenizer}}. This sheds some light on the limitations of the explanations we can give using these structures. The bias in the context of sub-word tokens loses the meaning it usually has and becomes a mere statistical preference for certain groupings of letters.

\subsection{Future Work}

In this study, we have established a foundational framework rooted in category theory for analyzing word embeddings and semantic spaces. Moving forward, there are several promising avenues for further research.

One direction for future exploration involves generalizing our framework to the over-category $\L / \mathcal{T}$, where $\mathcal{T}$ encompasses all texts. Extending our analysis to this broader context provides a rich mathematical framework for investigating biases inherent in training corpora. By developing robust criteria within this framework, we can effectively mitigate biases and enhance the reliability of word embeddings.

Moreover, our categorical approach offers a dual perspective, enabling the detection and quantification of biases present in textual data. By leveraging the categorical framework, future research can focus on developing methodologies to identify and characterize biases, leading to more comprehensive analyses of text data.

Exploring the connection between word embeddings, semantic spaces, and games through the lens of divergence as a combination of utility functions presents another intriguing avenue for future investigation. By delving into this connection, researchers can gain deeper insights into the underlying mechanisms driving word embeddings, offering new perspectives on their applications in natural language processing and machine learning.

Furthermore, integrating our categorical framework with other categorical approaches to machine learning algorithms holds promise for advancing our understanding of both word embeddings and machine learning techniques more broadly. By synthesizing diverse methodologies, future research can enrich our analyses and develop more comprehensive models for studying word embeddings and their implications.

Through these future research directions, we aim to extend the applicability and theoretical underpinnings of our categorical framework, fostering deeper insights into word embeddings and their role in natural language processing and machine learning.

\section{Acknowledgments}\label{sec:Ack}
	This work is partially supported by the TAILOR project, a project funded by the EU Horizon 2020 research and innovation programme under GA No 952215, and by the GUARDIA project, a project funded by Generalitat Valenciana GVA-CEICE project PROMETEO/2018/002

\newpage
\appendix{
	
	\section{Proofs}\label{sec:AppendixA}
	
	In this section, we present the proofs corresponding to the findings outlined in the manuscript.
	
	\subsection{Proof of Section \ref{sec:Spaces}}
	
	\begin{lemma}[Lemma \ref{lem:aux}]\label{lem:auxApp}
		Let $T$ be a text and $\L_{T}$ its syntax category and let $\textbf{2}=\{1,2\}$ be the category with two objects and only identity morphism. Let $W\colon\textbf{2}\to [0,1]$ be a functor of weights. Then colimit of the diagram of $F\colon\textbf{2}\to \L_{T}$ is an object $g^{*}=\colim{}^{W}F$ in $\L_{T}$ such that
		
		\begin{equation}\label{eq:colimmin}
			p(t\|g^{*}) = \min\left\{\frac{p(t\|F(1))}{W(1)},\frac{p(t\|F(2))}{W(2)},1 \right\} 
		\end{equation}
		
	\end{lemma}
	
	\begin{proof}
		This essentially follows from \cite[Definition 6]{CatSem} and \cite[Theorem 2]{CatSem}. Indeed, by Definition \ref{def:ecolim} the colimit satisfies, for any $t$ in $\L_{T}$
		
		\begin{equation}\label{eq:colimL}
			\L_{T}({\colim {}^{W}} F,t)\cong [\textbf{2}^{op}, [0,1]](W(-),\L_{T}(F(-),t)).
		\end{equation}
		The left-hand side of the equation translates to $p(t\|g^{*})$. By \cite[Definition 6]{CatSem} the right-hand side of the equation is a coend which, by \cite[Theorem 2]{CatSem}, is
		
		\begin{equation}
			\min\left\{\frac{p(t\|F(1))}{W(1)},\frac{p(t\|F(2))}{W(2)},1 \right\}.
		\end{equation}
	\end{proof}
	
	
	\begin{theorem}[Theorem \ref{thm:colimmaxprob}]\label{thm:colimmaxprobA}
		With the same hypothesis of the preceding lemma, the colimit of the diagram $g^{*}=\colim ^{W}F$ satisfies the following equation:
		
		\begin{equation}\label{eq:colimmaxprobA}
			g*= \max_{t\in \L_{T}}p(t\| g^{\pm})
		\end{equation}
		with $t=g^{-}wg^{+}\in \L_{T}^{2k+1}$ and
		
		\begin{equation}\label{eq:probW2VA}
			p(t\| g^{\pm})=\frac{C(t)}{\sum_{v\in\L_{T}^{1}}C(g^{-}vg^{+})}.
		\end{equation}
	\end{theorem}

	\begin{proof}
		The main idea is to use the universal property of the colimit over itself. Namely, let $g^{*}=\colim ^{W}F$, then 
		
		\begin{equation}
			1 = \L_{T}(g^{*},g^{*}) = \min\left\{\frac{p(g^{*}\|F(1))}{W(1)},\frac{p(g^{*}\|F(2))}{W(2)},1 \right\} 
		\end{equation} 
		
		which is greater or equal to $\L_{T}(g^{*},t)$ for any other $t$ in $\L_{T}$. Now, if there exists a $t'$ in $\L_{T}$ such that $p(t'\|g^{\pm})>p(g^{*}\|g^{\pm})$ then $C(t')>C(g^{*})$ (since the denominator is the same in both probabilities). This implies that $p(t'\|g^{-})>p(g^{*}\|g^{-})$ and $p(t'\|g^{+})>p(g^{*}\|g^{+})$. Thus yielding a contradiction
		
		\begin{equation}
			\min\left\{\frac{p(t'\|F(1))}{W(1)},\frac{p(t'\|F(2))}{W(2)},1 \right\} > \min\left\{\frac{p(g^{*}\|F(1))}{W(1)},\frac{p(g^{*}\|F(2))}{W(2)},1 \right\}. 
		\end{equation}

		Conversely, if an expression $t'$ satisfies equation \eqref{eq:colimmaxprobA} then $C(t')>C(t)$ for all $t=g^{-}vg^{+}$. This implies that $p(t'\|g^{-})>p(g^{*}\|g^{-})$ and $p(t'\|g^{+})>p(g^{*}\|g^{+})$ which in turn yields that 
		
		\begin{equation}
			\min\left\{\frac{p(t'\|F(1))}{W(1)},\frac{p(t'\|F(2))}{W(2)},1 \right\} > \min\left\{\frac{p(t\|F(1))}{W(1)},\frac{p(t\|F(2))}{W(2)},1 \right\}. 
		\end{equation}
		
		Hence, $t'$ satisfies the universal property of the colimit.
	\end{proof}

	\subsection{Proofs of Section \ref{sec:P_T}}
	
	\begin{theorem}[Theorem \ref{thm:MonoidalCat}]\label{thm:MonoidalCatA}
		The category $(\P_{T},\otimes, \L_{T}^{0})$ is indeed a monoidal category.
	\end{theorem}
	
	\begin{proof}
		There are two things to prove. The first is to show that $\P_{T}$ is a category with the operations defined above. The second is that $(\P_{T},\otimes, \L_{T}^{0})$ satisfies all the axioms of a monoidal category.
		
		Let's start with the first part. To show that $\P_{T}$ is a category we have to show that given two morphisms $f\colon X to Y$ and $g\colon Y\to Z$ their composition is also a morphism $h = g\circ f\colon X\to Z$ in $\P_{T}$. Specifically, we have to check that every entry of the product matrix is a positive number between $0$ and $1$. If the two morphisms are row-stochastic matrices, each entry of the matrix $h$ can be expressed as a sum
		
		\begin{equation}\label{eq:def:compmorphA}
			h_{ij}=\sum_{k} f_{ik}g_{kj}\leq \sum_{k} g_{kj}\leq 1
		\end{equation}
		
		since $f$ and $g$ are row-stochastic. Thus, their composition is again a morphism in $\P_{T}$\footnote{In fact, the product of row-stochastic matrices is again row-stochastic.}.

		Now, assume one of them is a probabilistic matrix, let's say $f$. Then, each entry of the product is weighted down by the inverse of the ceiling function of the number of the sum of the transition probabilities. This means that equation $\eqref{eq:def:compmorphA}$ becomes
		
		\begin{equation}\label{eq:def:compmorphceilA}
			h_{ij}= \frac{1}{\left\lceil\sum_{j}f_{ij}\right\rceil}\sum_{k} f_{ik}g_{kj} = \sum_{k} \frac{f_{ik}}{\left\lceil\sum_{j}f_{ij}\right\rceil} g_{kj}\leq 1.
		\end{equation}
		Thus, their composition is again a morphism in $\P_{T}$. Equations \eqref{eq:def:compmorphA} and \eqref{eq:def:compmorphceilA} imply the associativity of the composition of morphisms in $\P_{T}$: Finally, the identity morphism for each object $X$ in $\P_{T}$ is the identity matrix. Hence, $\P_{T}$ is a well-defined category.
		
		To show that the tuple $(\P_{T},\otimes, \L_{T}^{0})$ is a monoidal category we need to check that the tensor product we have defined satisfies associativity, that there are isomorphisms between the objects $X\otimes \L_{T}^{0}\cong X$ and $\L_{T}^{0}\otimes X\cong X$ and that the triangle and pentagon identities are satisfied.
		
		First of all, since the tensor product is just a concatenation of expressions of $T$ it is clear that the tensor product is associative. The $0$-th graded piece acts as concatenating $<eos>$ which is as concatenating nothing. Hence, it is the identity and we have $X\otimes \L_{T}^{0}\cong \L_{T}^{0}\otimes X\cong X$ given by $x\otimes <eos> = x = <eos>\otimes x$ for all $x\in X$.
		
		Lastly, the triangle and pentagon identities follow from the definition of the tensor product, the associativity condition, and the identity $x\otimes <eos> = x = <eos>\otimes x$. Thus, the tuple $(\P_{T},\otimes, \L_{T}^{0})$ is a monoidal category.
	\end{proof}
	
	\subsection{Proofs of Section \ref{sec:WE}}
	
	\begin{lemma}[Lemma \ref{lem:convergence}]\label{lem:convergenceA}
		The word embeddings $\{\mathcal{E}\colon P\to(X,C_{X})\}$ and $\{\mathcal{E}^{\prime}\colon P\to(X,C_{X})\}$ are equivalent if there is a sequence of divergences between them that convergences to $0$. In particular, if the sequences of divergences are decreasing, checking the equivalence of the embeddings is tantamount to verifying the conditions yielding the infimum for both sequences are the same.
	\end{lemma}
	
	\begin{proof}
		By definition, if the the sequences of embeddings are equivalent there exists an embedding $\{\mathcal{E}^{*}\colon P\to(X, C_{X})\}$ they both converge to. Since each of the embeddings have a sequence of divergences $\{D_{n}\}$ and $\{D_{n}^{\prime}\}$ 
		converging to $D_{*} $ and $D_{*}^{\prime}$ respectively , there is a divergence on $D(\mathcal{E}_{n}\| \mathcal{E}_{n}^{\prime})$ which converges to $0$.
		
		Lastly, if the sequences $\{D_{n}\}$ and $\{D_{n}^{\prime}\}$ are decreasing, since they are both bounded by $0$, then $ D_{*} $ and $D_{*}^{\prime}$ are the infimum of each sequences. Thus, verifying equivalence reduces to verifying that the conditions that minimize $ D $ are the same than the conditions that minimize $ D^{\prime} $. 
	\end{proof}

	
	\begin{theorem}[Theorem \ref{thm:main}]\label{thm:mainA}
		The GloVe and Word2Vec neural network embeddings are equivalent to metric MDS embeddings on vector spaces.
	\end{theorem}

	\begin{proof}
		By lemma \ref{lem:convergence}, to check the equivalence of word embeddings it suffices to check that the minimums of the loss functions are equivalent. This means verifying the configurations yielding the the infimum for both sequences are equivalent.
		
		For the GloVe embedding, the loss function of the algorithm is created to obtain vector representations satisfying equation \eqref{eq:GloVeCondition} (see \cite[Equation (5)]{pennington2014glove}). Thus the resulting word-vectors satisfy the following relation: $v_i^{T}\tilde{v_k}=(\log(P_{ik})+\log(P_{ki}))/2$. Hence, the distance between the vectors is given by $d^{2}_{ik}=G_{ii}+2G_{ik}+G_{kk}$. Now, the objective function of the GloVe algorithm, which can be found in \cite[Equation (8)]{pennington2014glove} or  \cite[Section 5.1]{hashimoto2016word}, is a weighted least squares regression model. Thus minimizing the objective function is equivalent to minimizing the absolute value of each of the summands since they are all squared.
		
		Each summand is of the form $v_i^{T}\tilde{v_k}+b_{i}+\tilde{b_k}-\log(X_{ik})$ where $X_{ik}$ is the co-occurrence count of word $i$ with word $k$. This is equivalent to minimizing the difference $v_i^{T}\tilde{v_k}-(\log(P_{ik})+\log(P_{ki}))/2$ by equations $(5)$ to $(7)$ of \cite{pennington2014glove}.
		
		We have reduced the problem to finding the minimum  									
		
		\begin{align}
			&\min\left\{ \left|v_{i}^{T}v_{j}-\frac{\log(P_{ij})+\log(P_{ji})}{2}\right|\right\}\quad\text{for all $i,j$}\\
			\Leftrightarrow &\min\left\{ |v_{i}^{T}v_{i}-v_{i}^{T}v_{j}-v_{j}^{T}v_{i}+v_{j}^{T}v_{j}-\left[\log(P_{ii})-(\log(P_{ij})+\log(P_{ji}))+\log(P_{jj}) \right] |\right\}\quad\text{for all $i,j$} \\
			\Leftrightarrow &\min\left\{|\|v_i-v_j\|^{2}- (d_{GL})^{2}_{ij}|\right\} \quad\text{for all $i,j$}.
		\end{align}
		
		The absolute value ensures that $\min\left\{|\|v_i-v_j\|^{2}- (d_{GL})^{2}_{ij}|\right\}$ is equivalent to $\min\left\{(\|v_i-v_j\|- (d_{GL})_{ij})^{2}\right\}$. This means that minimizing the objective function of the GloVe embedding is equivalent to minimizing the stress function of the MDS embedding for the dissimilarity matrix $d_{GV}$.
		
		The proof of the equivalence of the Word2Vec embedding and the MDS embedding for the dissimilarity matrix $d_{W2V}$ is analogous. The starting point is the divergence of the Word2Vec, by \cite[equation (13)]{pennington2014glove} embedding:
		
		\begin{equation}
			D_{W2V}= H(Q,S)=H(Q)+D_{KL}(Q\| S)
		\end{equation}
		
		where $Q$ is the matrix defined in equation \eqref{eq:EndoW2V} and $S$ is the probability distribution given by the softmax function:
		
		\begin{equation}
			S_{ij}=\frac{\exp(v_i^{T}v_j)}{\sum_{k}\exp(v_{i}^{T}v_k)}.
		\end{equation}
		
		The minimum is obtained when $Q=S$. Then since the inverse of the softmax is the natural logarithm (plus a constant) we get that this implies that the minimum is achieved when $d(_{W2V})_{ij}=\|v_i-v_j\|$.
		
		%

	\end{proof}

	\begin{cor}[Corollary \ref{cor:biasred}]\label{cor:biasredA}
		Let  $b_{i}(k,j)$ the bias of $w_{i}$ with respect to the pair $(w_{k},w_{j})$ and $S$. To remove this bias for the GloVe and Word2Vec embedding we need to equalize the quotients:

		\begin{equation}\label{eq:unbiasA}
			\frac{p_{kk}}{p_{ik}\cdot p_{ki}} = \frac{p_{jj}}{p_{ij}\cdot p_{ji}}
		\end{equation}
		with $p$ representing the probabilities of example \ref{exa:gv-w2v} $p=P$ in the case of GloVe and $p=Q$ in the case of Word2Vec. This ensures that after the embedding $d_{ik}=d_{ij}$ in both cases.
		
	\end{cor}

	\begin{proof}
		Letting the similarity matrix of the semantic spaces be

		\begin{equation}
			S_{ik} =\frac{p_{ii}\cdot p_{kk}}{p_{ik}\cdot p_{ki}}
		\end{equation} 
		
		we get that the bias of $w_{i}$ with respect to the pair $(w_{k},w_{j})$ and $S$ is $b_{i}(k,j)=S_{ik}/S_{ij}$. Note that to remove the bias we need to make $b_{i}(k,j)=1$ and that we can cancel $p_{ii}$ we get equation \eqref{eq:unbiasA}. Finally, by Theorem \ref{thm:main}, in both embeddings the distance matrix is given by:
		
		\begin{equation*}
			d_{ik}=\sqrt{\log\left(\frac{p_{ii}\cdot p_{kk}}{p_{ik}\cdot p_{ki}}\right)}.
		\end{equation*}
		
		If the $b_{i}(k,j)=1$ then $d_{ik}=d_{ij}$.
		
	\end{proof}
	
	\section{Limits and Colimits in Enriched Category Theory}\label{sec:AppendixB}
	
	In classical category theory, the limit of a diagram (a functor) $F\colon \J\to \C$ is an object $\lim F$ of $\C$  together with morphism to the objects $F(d_i)$ such that for objects $c$ of $\C$, there is a unique $f$ making the diagram commutes:
	
	\[\label{fig:limit}
	\begin{tikzcd}[row sep = large]
		& c \arrow[d, "f"] \arrow[ddr, bend left,"q_2"]   \arrow[ddl, bend right,"q_1",swap] & \\
		& \lim F \arrow[dr, "p_2"]  \arrow[dl,swap, "p_1"] & \\
		F(d_1) \arrow[rr,swap] & &  F(d_2) 
	\end{tikzcd}
	\]
	
	This can be rephrased using cones in the category $\C$. The base of the cone is given by the image of the functor $F$ and the apex (or summit) is given by an object in $\C$. In Figure \ref{fig:limit}, there are two cones, one with apex $c$ and one with apex $\lim F$. To be more precise: a cone with apex $c$ in $\C$ is a natural transformation from the constant functor of $c$, $\chi_{c}$ to the functor $F$. For any object $c$ in $\C$, the constant functor is a functor $\chi_{c}\colon\J\to \C$ that maps every object in $\J$ to $c$ and every morphism in $J$ to the identity morphism of $c$. Thus, what Figure \ref{fig:limit} is telling us is that the cone with apex $\lim F$ is terminal in the category of cones since any other cone has a unique morphism from its apex to $\lim F$ and the rest of morphisms emerge by composition.
	
	This can be summarized by the sentence: ``The limit of a diagram is the apex of the cone that is \textit{closest} to the base." The dual notion is the notion of a colimit. Being the dual of a limit means that we have a similar figure as Figure \ref{fig:limit} but with all its arrows reversed. This means that the cone with nadir (dual of apex) the colimit is initial in the category of cones under the diagram $F$. For a more extensive explanation and examples, we refer the reader to \cite[Chapter 3 ]{riehl2017category}.
	
	All this serves as a good intuition but it does not generalize well to the enriched setting. The enriched counterparts of limits and colimits are called \textbf{weighted limits} and \textbf{weighted colimits}. To generalize the above definitions we need to look at limits/colimits via their universal properties. For this, notice that a way to encode the \textit{set} of cones with apex $c$ and base $F$ is the set
	
	\begin{equation}\label{eq:cones}
		[\J, \Set](*(-),\C(c,F(-))).
	\end{equation}
	
	In equation \eqref{eq:cones} $[\C,\Set]$ is the functor category from $\C$ to the category of sets: $\Set$. This category has as objects functors and as morphisms natural transformations between functors. Then, $[\J,\Set](*(),\C(c,F())$ represents a natural transformation from the constant functor $*(-)$ at the terminal object (of the category $\Set)$ to the functor $\C(c,F(-))$. A natural transformation $*(-)\Rightarrow \C(c,F(-))$ is a way of specifying and arrow $c\to F(i)$ for each $i$ in $\J$. The naturality condition implies that the set of these arrows satisfies the necessary commutativity conditions to form a cone over $c$ with base $F$. Thus, by definition of the limit, we have the following isomorphism:
	
	\begin{equation}\label{eq:univproplim}
		\C(c,\lim F)\cong [\J, \Set](*(-),\C(c,F(-))).
	\end{equation}
	
	Equation \eqref{eq:univproplim} is merely a mathematical way of stating that ''Specifying a morphism from any apex $c$ of $\C$ to the limit $\lim F$ (LHS) is equivalent to specifying a set of arrows from $c$ to $F(i)$ for all $i$ of $\J$ (RHS)".
	
	The advantage of using this notation is twofold. First of all, stating the universal property of the colimit is immediate:

	\begin{equation}\label{eq:univpropcolim}
		\C(\colim F,c)\cong [\J^{op}, \Set](*(-),\C(F(-),c)).
	\end{equation}
	
	The second is that it allows us to generalize it to the enriched setting. If $\C$ is a $\V$-enriched category, then the space of morphism between two objects $\C(c,d)$ is no longer a set, it is an object in $\V$. To take this into account we substitute the constant functor $*\colon J\to \Set$ by a \textbf{$\V$-functor of weights} $W\colon\J\to\V$. This yields the definition of a \textbf{weighted limit of $F$ by $W$}:
	
	\begin{definition}\label{def:ecolim}
		Given a $\V$-functor (or $\V$-diagram) between two $\V$-enriched categories $F\colon\J\to\C$ and a weight $\V$-functor $W\colon\J\to\V$ we define the \textbf{weighted limit of $F$ by $W$}, if it exists, as an object $\lim^{W}F$ satisfying the following isomorphism in $\V$:

		\begin{equation}\label{eq:eunivproplim}
			\C(c,{\lim {}^{W}} F)\cong [\J, \V](W(-),\C(c,F(-))).
		\end{equation}
		
		Dually, to define the \textbf{weighted colimit of $F$ by $W$} we need to modify the functor of weights with domain in the opposite category of $\J$, $W\colon\J^{op}\to\V$. With that, the \textbf{weighted colimit of $F$ by $W$}, if it exists, as an object $\colim^{W}F$ satisfying the following isomorphism in $\V$:
		
		\begin{equation}\label{eq:eunivpropcolim}
			\C({\colim {}^{W}} F,c)\cong [\J^{op}, \V](W(-),\C(F(-),c)).
		\end{equation}
		
	\end{definition}
	
	These are highly abstract and technical definitions, but in practice, what we are achieving with the functor $W$ is to specify some objects and morphisms in $\V$. This restricts the possible limit/colimits to those that satisfy certain properties (given by the functor $W$). A more detailed exposition can be found in \cite[Chapter 7]{riehl2014categorical}.
	
	\begin{re}
		There are several technical requirements to make this work. One of them is that the base of enrichment $\V$ should be symmetric monoidally closed to ensure that morphisms in $\V$ can also be seen as objects in $\V$. This is the case when the base of enrichment is the category $[0,1]$ as seen in \cite[Section 2.1]{CatSem}.
		
	\end{re}
}

\bibliographystyle{alpha}
\newpage
%

\bibliography{Explainability_W2V.bib}

\end{document}